\documentclass{article}

\PassOptionsToPackage{numbers, compress}{natbib}
\usepackage[preprint]{neurips_2026}

\usepackage{microtype}
\usepackage{graphicx}
\usepackage{subcaption}
\usepackage{booktabs} 

\usepackage{hyperref}

\usepackage{amsmath,amsfonts,bm}

\def\eqref#1{equation~\ref{#1}}

\def\1{\bm{1}}

\DeclareMathAlphabet{\mathsfit}{\encodingdefault}{\sfdefault}{m}{sl}
\SetMathAlphabet{\mathsfit}{bold}{\encodingdefault}{\sfdefault}{bx}{n}

\usepackage{amsmath}
\usepackage{amssymb}

\usepackage{mathtools}
\usepackage{algorithm}
\usepackage{algorithmic}
\usepackage{amsthm}

\usepackage{colortbl}
\usepackage{xcolor}
\arrayrulecolor{black}
\usepackage{hyperref}
\usepackage{url}
\usepackage{booktabs}
\usepackage{enumitem}
\usepackage{amsmath}
\usepackage{amsthm}
\usepackage{amssymb}
\usepackage{natbib}
\usepackage{xspace}
\usepackage{adjustbox}
\usepackage{multirow}
\usepackage{wrapfig}
\usepackage{pifont}
\usepackage{sidecap}
\usepackage{graphicx}
\usepackage{colortbl} 
\usepackage{caption}        
\usepackage{soul}
\usepackage{makecell}  

\usepackage[capitalize,noabbrev]{cleveref}

\theoremstyle{plain}
\newtheorem{theorem}{Theorem}[section]
\newtheorem{proposition}[theorem]{Proposition}

\theoremstyle{definition}

\theoremstyle{remark}

\usepackage[textsize=tiny]{todonotes}
\definecolor{midblue}{rgb}{0.21,0.49,0.74}
\definecolor{alicegreen}{rgb}{0.90,1.0,0.90}
\definecolor{alicered}{rgb}{1.0, 0.90, 0.90}
\definecolor{aliceblue}{rgb}{0.93, 0.93, 1.0}

\definecolor{lightYellow}{HTML}{FFF8DC}  
\definecolor{cornellred}{rgb}{0.7, 0.11, 0.11}
\definecolor{cadmiumgreen}{rgb}{0.0, 0.42, 0.24}
\definecolor{darkblue}{rgb}{0.83, 0.89, 0.97}
\definecolor{Red7}{rgb}{0.941, 0.243, 0.243}
\definecolor{Green7}{RGB}{55, 178, 77} 
\definecolor{Blue9}{rgb}{0.098,0.3,0.9}
\definecolor{customyellow}{HTML}{FFFCF5}
\definecolor{customred}{HTML}{FFF9F9}
\definecolor{customblue}{HTML}{F9FCFF}
\definecolor{customgreen}{HTML}{F9FFF9}
\definecolor{ForestGreen}{rgb}{0.13, 0.55, 0.13}
\definecolor{ForestBlue}{rgb}{0.13, 0.45, 0.80}
\definecolor{ForestRed}{rgb}{0.92, 0.2, 0.2}

\usepackage{pifont}

\usepackage[utf8]{inputenc} 
\usepackage[T1]{fontenc}    
\usepackage{hyperref}       
\usepackage{url}            
\usepackage{booktabs}       
\usepackage{amsfonts}       
\usepackage{nicefrac}       
\usepackage{microtype}      
\usepackage{xcolor}         

\title{Dual-Dimensional Consistency: Balancing Budget and Quality in Adaptive Inference-Time Scaling}

\author{%
  Rongman Xu\textsuperscript{*} \quad
  Hang Yan\textsuperscript{*} \quad
  Yifei Li\textsuperscript{*} \quad
  Tianzhe Zhao \quad
  Yanrui Wu \quad
  Bo Li \quad
  \vspace{0.07cm}\\
  \texttt{rongmanxu@stu.xjtu.edu.cn} \quad \texttt{hyan@stu.xjtu.edu.cn} \quad \texttt{yifeilee@stu.xjtu.edu.cn}
  \vspace{0.2cm}\\ 
  Xi'an Jiaotong University \\
}

\begin{document}

	\maketitle
	\begingroup
	\renewcommand\thefootnote{}
	\footnotetext{$^{*}$ Equal contribution.}
	\endgroup
	
	
	\begin{abstract}
        Large Language Models (LLMs) have demonstrated remarkable abilities in reasoning. However, maximizing their potential through inference-time scaling faces challenges in trade-off between sampling budget and reasoning quality. Current strategies remain inefficient as they typically treat sampling width and depth as orthogonal objectives, where width consensus methods risk reinforcing hallucinations, while depth pruning mechanisms prematurely truncate complex yet valid reasoning chains.
        Therefore, we propose Dual-Dimensional Consistency (DDC), a unified framework that bridges path quality with adaptive termination. By coupling Confidence-Weighted Bayesian protocol with a Trend-Aware Stratified Pruning, our method ensures that computational resources are concentrated on high quality reasoning paths, filtering hallucinations while accelerating consensus. Evaluations across five benchmarks demonstrate that this approach reduces token consumption by over $10\times$ while maintaining or exceeding the accuracy of strong baselines across various LLMs.
	\end{abstract}
	
	\section{Introduction}
    \label{introduction}
	\vspace{-0.2cm}
	Large Language Models (LLMs) have achieved remarkable performance across a wide range of tasks due to extensive pre-training and post-training \citep{devlin2019bertpretrainingdeepbidirectional,radford2021learningtransferablevisualmodels}. However, these training-centric approaches incur substantial computational costs \citep{raffel2023exploringlimitstransferlearning,stiennon2022learningsummarizehumanfeedback}.
    To address this, inference-time scaling  \citep{wei2023chainofthoughtpromptingelicitsreasoning,wang2023selfconsistencyimproveschainthought,wu2025inferencescalinglawsempirical,xu2025phidecodingadaptiveforesightsampling} has emerged as an efficient alternative, significantly saving computational costs while maintaining and even enhancing performance  \citep{deng2025testtimemodeladaptationquantized}.

    To implement inference-time scaling, traditional static strategies like Self-Consistency  \citep{wang2023selfconsistencyimproveschainthought}, as illustrated in \cref{fig:draft}(a), allocate a fixed budget like $B=512$ to sample reasoning paths for majority voting to reach a consensus. This rigid paradigm is inefficient, as simple queries consume the same excessive resources as complex ones \citep{NEURIPS2024_3c1e1fdf}. Fundamentally, this approach attempts to compensate for low path quality through massive samplings, inefficiently prioritizing quantity over the correctness of reasoning.
    
    \begin{figure*}[t] 
      \centering
      \includegraphics[scale=0.3]{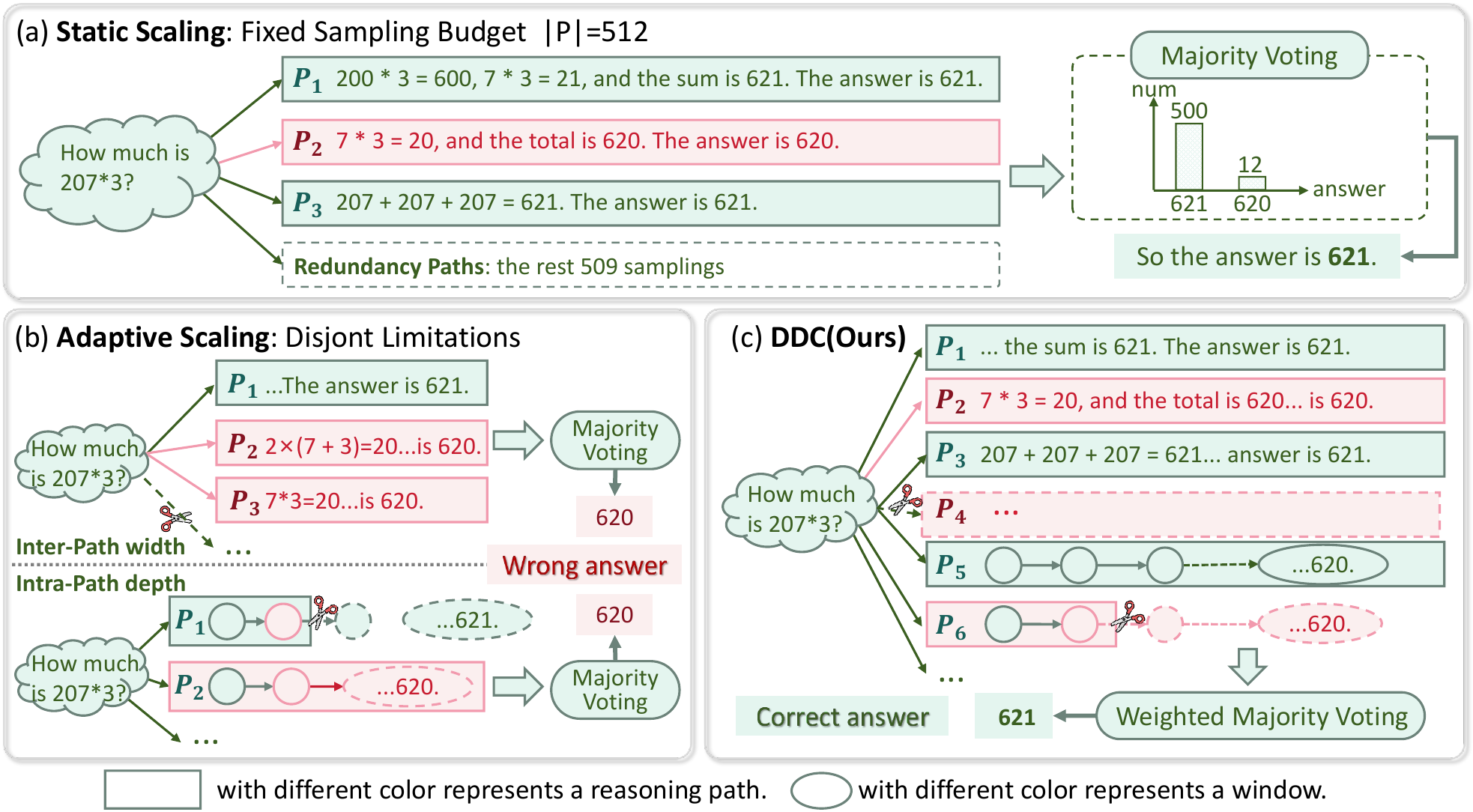}
      \caption{
        Comparisons of different methods. (a) is traditional static scaling method with fixed sampling budget of 512. 
        (b) is the current adaptive scaling method of inter-path sampling width and intra-path generation depth. 
        (c) is our method: a joint framework.
      }
      \label{fig:draft}
    \end{figure*}
    
    To address these limitations, recent studies explore adaptive inference strategies. Based on the classification in \cref{fig:draft}(b), these methods optimize inference via two dimensions: inter-path width and intra-path depth. Width-based approaches manage sampling budget by terminating the entire sampling process upon early consensus. However, as models often exhibit hallucinations, assuming consistency equals correctness risks unsafe termination on consensus where paths agree on incorrect answers. Depth-based methods focus on quality of each path \citep{sharma2025thinkjustenoughsequencelevel}, pruning invalid paths via step-wise verification like token confidence \citep{snell2024scalingllmtesttimecompute,huang2024largelanguagemodelsselfcorrect}. However, current pruning strategy often leads to sampling redundancy for simple queries and incorrect pruning of valid complex paths \citep{li2025modelinguncertaintytrendstimely}. Crucially, prevailing works treat width and depth as orthogonal objectives  \citep{aggarwal2023letssamplestepstep}, neglecting the intrinsic trade-off between sampling budget, generating excessive paths like \cref{fig:draft}(b) and path quality, which leads to suboptimal resource allocation.
    
    To bridge this gap, we introduce Dual-Dimensional Consistency(DDC), a unified framework designed to scale adaptive inference by jointly optimizing path quality and sampling budget. To minimize sampling budget, we propose a Confidence-Weighted Bayesian Termination protocol. This mechanism integrates path-level quality into the consensus to decide the final answer, ensuring the entire sampling terminates early only when the evidence is both consistent and high confidence. Simultaneously, to refine path quality, we introduce Trend-Aware Stratified Pruning. Modeling each reasoning process as a dynamic signal to disentangle stable trends from stochastic noise \citep{li2025performanceanalysismomentummethod,yan2025murmomentumuncertaintyguided}, this pruning strategy adaptively discards hallucinations while preserving complex but correct paths. By unifying these components, DDC achieves balance between width of sampling and depth of individual paths, dynamically allocating resources to the most promising reasoning chains.
    
    Evaluations across five challenging reasoning benchmarks validate the efficacy of our approach. For instance, on AIME25 using Qwen3-4B, DDC achieves a 15.6\% accuracy gain over Self-Consistency while reducing token consumption by approximately 27$\times$. On average, our method reduces computational costs by over 10$\times$ compared to strong baselines, delivering superior performance with significantly reduced budget.
    
    The main contributions are highlighted as follows: \textbf{(1) Unified Optimization Framework.} We propose a framework that couples inter-path width with intra-path depth to maximize efficiency by jointly calibrating termination and pruning based on shared quality signals. \textbf{(2) Confidence-Weighted Bayesian Termination.} We formulate early stopping as a Bayesian sequential decision problem that weights sampling evidence with path quality, ensuring consensus and preventing unsafe termination. \textbf{(3) Trend-Aware Stratified Pruning.} We develop a dynamic pruning strategy that distinguishes reasoning trends from noise, allowing the model to filter hallucinations while preserving valid reasoning paths. \textbf{(4) State-of-the-Art Performance.} Extensive evaluations show that DDC redefines the efficiency-accuracy trade-off, achieving over $10\times$ efficiency gains while maintaining performance parity.
    
	\section{Related Works}
	
    \textbf{Inference-Time Scaling.}
    Recent advances in scaling have shifted from parameter scaling to inference-time scaling, revealing a inference-time scaling law  \citep{snell2024scalingllmtesttimecompute, brown2024largelanguagemonkeysscaling, wu2025inferencescalinglawsempirical} where increasing computational resources during inference consistently enhances performance \citep{hao2023reasoninglanguagemodelplanning}. These methods are primarily categorized into training-based and training-free approaches.
    Training-based methods leverage reinforcement learning with verifiable rewards (RLVR) to induce long form reasoning patterns, as demonstrated in models such as DeepSeek-R1  \citep{deepseekai2026deepseekr1incentivizingreasoningcapability} and other reasoning-optimized frameworks  \citep{ye2025limoreasoning, zuo2025ttrltesttimereinforcementlearning}. 
    Training-free methods optimize reasoning during inference, typically through searching-based or sampling-based strategies. Searching-based methods \citep{yao2023treethoughtsdeliberateproblem,hao2023reasoninglanguagemodelplanning,xu2025phidecodingadaptiveforesightsampling}, refine answers by exploring multiple paths, while sampling-based methods \citep{wang2023selfconsistencyimproveschainthought,kang2025scalablebestofnselectionlarge} generate and select from multiple candidate solutions. 
    \begin{figure*}[ht]
      \centering
      \includegraphics[scale=0.35, keepaspectratio]{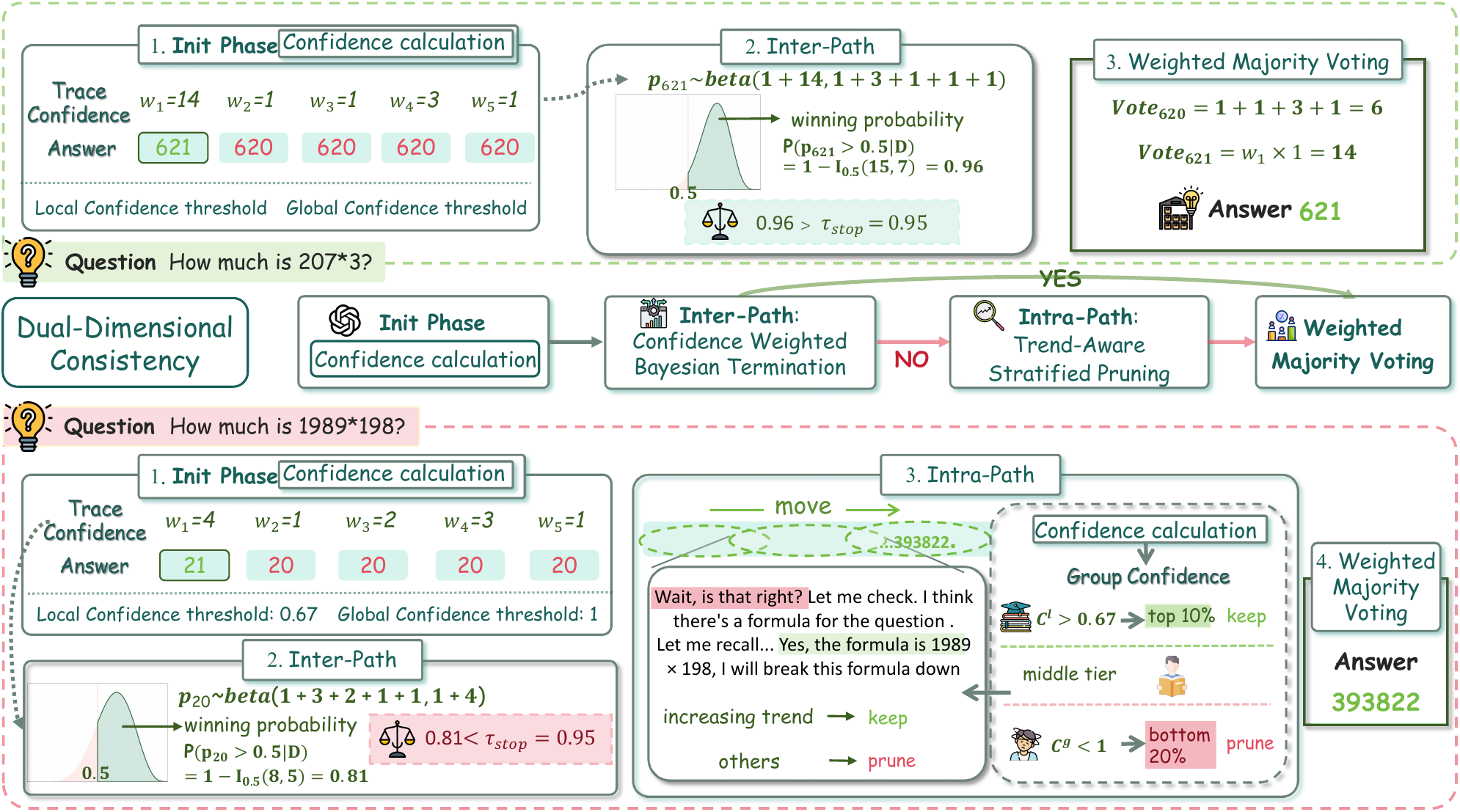}
      \caption{
        Illustration of the DDC inference process. Init Phase Confidence Calculation computes initial metrics to guide resource allocation. Inter-Path Bayesian Termination updates the Beta distribution using global
        path confidence as weights, triggering early stopping upon consensus. If generation continues, Intra-Path Trend-Aware Pruning monitors reasoning dynamics via a sliding window to dynamically discard invalid paths.
      }
      \label{fig:draft2}
    \end{figure*}
    These methods improve reasoning without retraining, though they often involve high computational costs  \citep{madaan2023selfrefineiterativerefinementselffeedback, lan2024criticevalevaluatinglargelanguage, li2025performanceanalysismomentummethod}.

    \textbf{Adaptive Inference and Verification.} Recent works try to dynamically adjust compute resources based on query difficulty. Inter-Path methods adaptively sample and aggregate paths \citep{li2024escapeskyhighcostearlystopping}, using techniques like early stopping and difficulty aware policies \citep{wang2025makepennycountdifficultyadaptive,kim2026reliabilityawareadaptiveselfconsistencyefficient} to reduce redundant samplings. Intra-Path methods optimize the depth of reasoning by guiding or pruning searches \citep{aggarwal2023letssamplestepstep} and depth control strategies MUR  \citep{yan2025murmomentumuncertaintyguided}) that avoid overthinking and hallucinations by monitoring uncertainty. These strategies treat inter-path width and intra-path depth as orthogonal objectives, neglecting the connection between quality and global budget.
    
    \textbf{Reasoning path Analysis.}
    Monitoring the generation process is essential for reliable scaling. While token-level uncertainty can signal confidence  \citep{quevedo2024detectinghallucinationslargelanguage,zhang2025tokurtokenleveluncertaintyestimation}, raw scores are often too noisy for precise pruning  \citep{yan2025murmomentumuncertaintyguided}.
    Drawing inspiration from signal processing  \citep{li2025performanceanalysismomentummethod}, specifically Singular Spectrum Analysis, we denoise these paths. This allows us to rescue hard but correct paths that exhibit temporary confidence dips, a nuance overlooked by current static pruning and selection methods.

    \section{Problem Definition}

    LLM reasoning is generation task: given query $x$, the model samples path $y$, tokens sequence $\{ y_t \}_{t=1}^T$ according to the distribution $P_\theta(y \mid x)$, leading to a answer $u \in \mathcal{U}$. Let $u_{\text{gt}}$ be the ground truth.
    
    Static Scaling identifies the leader $u^* = \arg\max_{u \in \mathcal{U}} \Phi(u; S)$ from  $n$ independent paths $\mathcal{S} = \{y^{(p)}\}_{p=1}^{n}$, 
    where $\Phi(u; \mathcal{S})$ is the consensus measuring the support for answer $u$ within $\mathcal{S}$. 
    
    Adaptive Scaling defines this as a dual decision: (1) Inter-Path Termination in Width: Terminate sampling when the probability that $u^*$ is the final answer exceeds a reliability threshold $\tau_{\text{stop}}$. (2) Intra-Path Pruning in Depth: Prune $y^{(p)}$ at step $t$ if it is unlikely to reach $u_{\text{gt}}$. The objective is to construct a dynamic set $S_{\text{dyn}}$ to minimize $\sum_{y \in S_{\text{dyn}}} |y|$ subject to $P(\Psi(S_{\text{dyn}}) = u_{\text{gt}}) \ge \delta_{\text{target}}$, 
    where $\Psi(\mathcal{S}_{dyn}) = u^*$ represents the final consensus answer, and $\delta_{\text{target}}$ is accuracy of other scaling method. 
	
	\section{Methodology}
    \label{method}
    To balance sampling budget and path quality, we propose DDC, an adaptive framework utilizing a multi-granular Confidence Metrics.

    As illustrated in \cref{fig:draft2}, DDC proceeds as follows: Init Phase Confidence Calculation first deploys a budget $B_{\text{init}}$ to profile the initial confidence. Inter-Path Bayesian Termination then assesses, executing an early exit if the leading answer exhibits statistical sufficiency. If uncertainty remains, Intra-Path Trend-Aware Pruning expands generation to the full budget $B$, actively monitoring reasoning with a sliding window in real-time to discard low quality trajectories. Finally, Consensus derives the answer through confidence-weighted majority voting on the set of verified paths.
    
    \subsection{Multi-Granular Confidence Calculation Metrics}
    To quantify reasoning quality, we adopt the confidence measurements proposed by  \citep{fu2025deepthinkconfidence}, establishing a multi-granular confidence metrics that dynamically captures step-wise dynamics from token-level to path-level.
    
    \begin{wrapfigure}[9]{r}{0.4\textwidth}
        \centering
        \vspace{-1em}
       \includegraphics[width=1\linewidth]{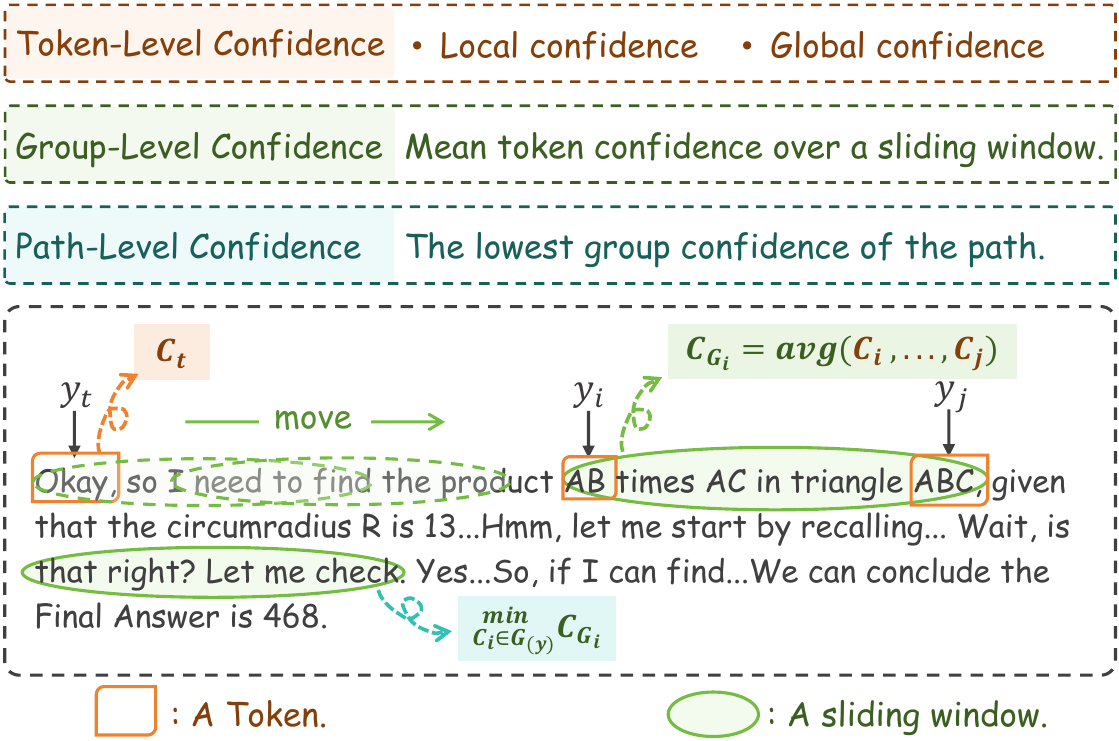}  
       \vspace{-1.5em}
        \caption{Illustrating the  multi-granular confidence.}
        \label{metrics}
    \end{wrapfigure} 
    
    As illustrated in \cref{metrics}, our metrics is defined as follows:
    \paragraph{Token-Level Confidence.}
    We measure the model's certainty at generation step $t$ by two metrics: \textbf{(1)Local Confidence $C_t^l$} is defined as the maximum token probability:
    $C_t^l = P_t(1)$.
        \textbf{(2)Global Confidence $C_t^g$} captures richer distributional information  from the top-$k$ tokens by computing their negative average log-probability:
    \begin{equation}
            C_t^g = -\frac{1}{k} \sum_{j=1}^{k} \log P_t(j),
    \end{equation}
    where lower confidence reflects uncertain predictions.
    
    \paragraph{Group-Level Confidence.}
    For a sliding window of tokens $G_i$ starting at step $i$, the group confidence $C_{G_i}$ is defined as the mean token confidence:
    \begin{equation}
    \label{group_confidence}
        C_{G_i} = \frac{1}{|G_i|} \sum_{t = i}^{i+|G_i|-1} C_t,
    \end{equation}
    which captures lower value in sustained uncertain reasoning.
    
    \paragraph{Path-Level Confidence.}
    The overall quality of a path $y$ is its lowest group confidence:
    \begin{equation}
    \label{trace_confidence}
        C_{\text{path}}(y) = \min_{G_i \in \mathcal{G}(y)} C_{G_i},
    \end{equation}
    where $\mathcal{G}(y)$ is the set of all sliding windows in $y$. This metric more sensitively identifies the most uncertain segments in the reasoning path, which enables more effective elimination of low-quality paths while providing more discriminative weight signals for voting.

     \textbf{Init Phase.} We generate a set of $B_{\text{init}}$ paths. For each path, we compute the aforementioned confidence metrics to establish a query-specific quality baseline, forming the initial evidence.
    \subsection{Inter-Path: Confidence-Weighted Bayesian Termination}
    To minimize sampling redundancy without the overhead of high-dimensional Dirichlet integration  \citep{aggarwal2023letssamplestepstep}, we reframe consensus as a binary competition between the leading candidate $u^*$ and the aggregate alternatives, simplifying the model to a Beta distribution, enabling efficient real-time termination.

    \textbf{Bayesian Sequential Decision Formulation.} 
    We frame the adaptive scaling of sampling width as a Bayesian sequential decision problem. After the init phase and given observed data $\mathcal{D}$, we model the uncertainty of the current leading answer $u^*$'s winning probability $p_u$ as a Beta posterior: $p_u \mid \mathcal{D} \sim \text{Beta}(\alpha_u, \beta_u)$, where $\alpha_u$ and $\beta_u$ accumulate the confidence of paths supporting the leader and competitors, respectively, enabling rigorous quantification of termination risks.

    
    \textbf{Adaptive Stopping Rule.} 
    To determine the termination point, we utilize the cumulative distribution function (CDF) of $\text{Beta}(\alpha_u, \beta_u)$: 
    \begin{equation}
        I_\gamma(\alpha_u,\beta_u) = \frac{\int_0^\gamma t^{\alpha_u-1}(1-t)^{\beta_u-1}dt}{\int_0^1 t^{\alpha_u-1}(1-t)^{\beta_u-1}dt},
    \end{equation}
     allows us to quantify the probability that the answer $u$ occupies a sufficient majority portion $\gamma$ as the ratio of the partial integral to the complete normalizing integral. Generation terminates when the posterior confidence in the leading answer $u^* = \arg\max_{u} \alpha_u$ implies that the probability of $p_{u^*}$ exceeding the majority threshold $\gamma$ is high enough:
    \begin{equation}
    \mathbb{P}(p_{u^*} > \gamma \mid \mathcal{D}) = 1 - I_{\gamma}(\alpha_{u^*}, \beta_{u^*}) > \tau_{\text{stop}}.
    \end{equation}
    We set $\gamma = 0.5$ to ensure the formation of an absolute majority rather than a relative majority and $\tau_{\text{stop}} = 0.95$ for high statistical significance. This rule provides a robust termination criterion, minimizing sample size while guaranteeing that the result is backed by an absolute advantage.
    
    \textbf{Confidence-Weighted Parameter Update.} 
    To further enhance the reliability of this early exit, we introduce a confidence-weighted Bayesian update mechanism. Instead of treating each sample equally, we assign a weight $w_i = C_{\text{path}}(y^{(i)})$ to each reasoning path. This integrates path quality directly into the Bayesian framework, ensuring that termination is driven by rigorous reasoning evidence rather than mere frequentist consensus. By prioritizing high integrity paths, this mechanism prevents unsafe exits on spurious majorities while adaptively minimizing sampling budget.
    
    For a candidate answer $u$, we define the accumulated evidence $\alpha_{u}$ in support of $u$ and counter-evidence $\beta_{u}$ in support of all competitors as follows:
    \begin{equation}
    \alpha_{u} = 1 + \sum_i \mathbb{I}\{\text{ans}_{i} = u\} \cdot w_{i}, 
    \end{equation}
    \begin{equation}
    \quad \beta_{u} = 1 + \sum_{j} \mathbb{I}\{\text{ans}_{j} \neq u\} \cdot w_{j},
    \end{equation}
    where the constant $1$ represents a Laplace prior, corresponding to a uniform prior $\text{Beta}(1,1)$ to account for initial uncertainty. This ensures that the model maintains a non-informative initial state before any evidence is collected and prevents overconfidence during the initial sampling stages.
    Detailed analysis theory is in Appendix~\ref{sec:theory}.

    \subsection{Intra-Path: Trend-Aware Stratified Pruning}
    
    To optimize path quality for adaptive inference scaling, it is insufficient to simply increase the number of samples, as indiscriminate scaling introduces excessive noise and computational waste. We propose Trend-Aware Stratified Pruning to ensure that only high quality reasoning paths continue to generate tokens, while low quality paths are pruned early to prevent them from polluting the final answer consistency vote. This hierarchical gating mechanism transitions from static scalar thresholds to a dynamical analysis of the reasoning process, allowing the model to adaptively allocate compute based on the evolving trend of each reasoning path. We categorize paths into three tiers based on the trade-off between computational cost and information exploitation.
    
    First, we caculate the group confidence $C_{G_m}$ using local confidence. Paths maintaining confidence higher than top-10\% of the init distribution bypass subsequent analysis and proceed with reasoning on the insight that high local confidence correlates with correct reasoning  \citep{fu2025deepthinkconfidence}. Then we caculate $C_{G_m}$ utilizing global confidence with more comprehensive information to safely identify paths and prune the bottom-20\% immdiately. The critical challenge lies in the middle tier paths, which contain a mixture of valid difficult reasoning and hallucinations. Static averaging fails here, as it cannot distinguish between a temporary confidence dip during hard problem-solving and the chaotic, stagnant confidence characteristic of a hallucination.

    To resolve this, we map the reasoning process into a dynamical phase space. We construct a position-velocity state vector $\mathbf{z}_{t} \in \mathbb{R}^{2}$ for each step $t$ within a window of size $k$:
    \begin{equation}
    \mathbf{z}_{t} = \begin{bmatrix}
    \displaystyle \frac{C_t^g - \mu_c}{\sigma_c} \\[8pt]
    \displaystyle \frac{C_t^g - C_{t-1}^g}{\sigma_c}
    \end{bmatrix},
    \end{equation}
    where $C_t^g$ is the global confidence, $\mu_c$ is the window mean and $\sigma_c$ is the window's standard deviation. This normalization ensures scale-invariance, meaning the analysis focuses on the relative evolution trend of confidence specifically rather than using absolute confidence value. And the two components represent the relative position and relative velocity of confidence  respectively. We represent the path as a matrix $\mathbf{Z} \in \mathbb{R}^{k \times 2}$ and perform an eigendecomposition on the lag-covariance matrix $\mathbf{\Sigma} = \frac{1}{k-1} \mathbf{Z}^\top \mathbf{Z}$ to obtain eigenvalues $\lambda_1 \geq \lambda_2$ and eigenvector $\vec{v_1}$. $\lambda_1$ reflects the dispersion along the primary direction of change, while $\lambda_2$ reflects the orthogonal dispersion. The difference $\lambda_1 - \lambda_2$ measures the directionality of the path: a large difference indicates trend like a path recovering from a dip, whereas $\lambda_1 \approx \lambda_2$ indicates isotropic, stochastic noise typical of hallucinations. We define the \textbf{Structural Instability Score} $\mathcal{R}$ to penalize both unguided noise and negative momentum:
    \begin{equation}
    \mathcal{R} = \left( 1 - \frac{\lambda_1 - \lambda_2}{\lambda_1 + \lambda_2} \right) + \eta \cdot \mathbb{I}(\text{align} < 0) \cdot \text{align}^2,
    \end{equation}
    where $\text{align} = |v_{1,x}| \cdot \text{sign}(\bar{v})$, $|v_{1,x}|$ represents the alignment of the path's primary trend with the confidence axis, and $\bar{v}$ is the mean normalized velocity. 
    The $\text{align}$ metric captures both the directionality of the path by the eigenvector and the sign of the velocity trend.
    Crucially, we employ the mean velocity $\bar{v}$ as the directional pivot for $\text{align}$ because the rate of change in confidence serves as a leading indicator of reasoning failure rather than position.
    
    If the dominant trend is a sustained decrease in confidence where $\bar{v}<0 $, we imposes a quadratic penalty. This ensures that we salvage difficult but recovering paths where confidence dips then rebounds while aggressively pruning paths that are sustained degenerating.

    Finally, we establish the pruning criterion based on $\mathcal{R}$. While static percentile thresholds suffice for filtering clear-cut confidence extremes as in our initial gating, they are ill-suited for this middle tier where the baseline instability varies significantly across queries. Imposing a fixed quota here risks indiscriminately pruning valid paths in complex tasks or retaining noise in trivial ones.
    To ensure robust adaptivity, we formulate pruning as a query-specific outlier detection problem. We derive the dynamic threshold $\tau_{\text{risk}}$ using Tukey Fences  \citep{tukey1977exploratory}, which defines a statistical boundary based on the dispersion of the init population:
    \begin{equation}
        \tau_{\text{risk}} = Q_3 + 1.5 \cdot \text{IQR},
    \end{equation}
    where $Q_3$ denotes the third quartile and $\text{IQR} = Q_3 - Q_1$. By calibrating the threshold against the population's interquartile range, this method automatically scales with the structural variance of the reasoning process, effectively isolating anomalous paths $\mathcal{R} > \tau_{\text{risk}}$ while accommodating the intrinsic difficulty of the task.
    
    \subsection{Final Consensus}
    \label{Consensus}
    The final answer is determined by applying weighted majority voting on the set of all verified paths $\mathcal{S} = \mathcal{Y}_{init} \cup \mathcal{Y}_{valid}$.
    \begin{equation}
    \hat{y} = \arg\max_{u} \sum_{y^{(i)} \in \mathcal{S}} \mathbb{I}\{\text{ans}_i = u\} \cdot C_{\text{path}}(y^{(i)}),
    \end{equation}
    where higher weighted value denotes higher quality. It ensures that the final vote is not swayed by the quantity of generated paths, but by their quality.
    
	\begin{wrapfigure}[12]{r}{0.4\textwidth}
        \centering
        \vspace{-5em}
       \includegraphics[width=1\linewidth]{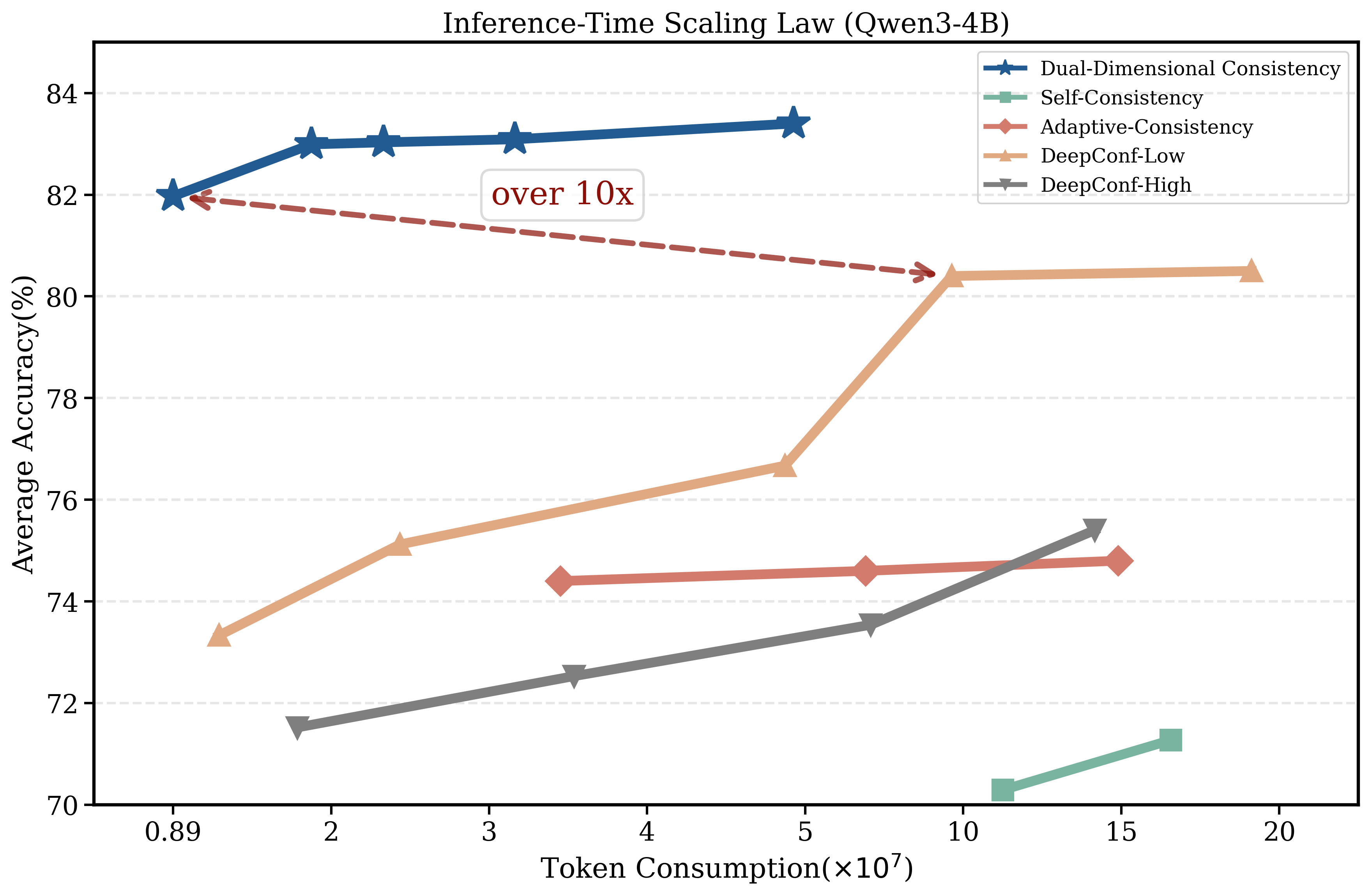}  
        \caption{Inference-time scaling law on Qwen3-4B-Instruct. The horizontal axis denotes the token consumption, while the vertical axis represents the average accuracy on 5 benchmarks.}
        \label{inference_time_scaling}
    \end{wrapfigure} 	
	
	\section{Experiments}
    \label{sec:experiments}
    
    \subsection{Experimental Setup}
    \label{sec:experiments-setup}
    
    \textbf{Datasets and Models.} We evaluate on five challenging reasoning benchmarks: MATH-500 \citep{hendrycks2021measuringmathematicalproblemsolving}, AMC23, AIME24, AIME25, and GPQA-diamond \citep{rein2023gpqagraduatelevelgoogleproofqa}. To make fair compraison, we evaluate on the same LLMs following previous works \citep{fu2025deepthinkconfidence,yan2025murmomentumuncertaintyguided}, specifically on Qwen 3 series and DeepSeek-R1-0528-Qwen3-8B. Detailed descriptions are provided in Appendix \ref{app:datasets-models}.

    \textbf{Baselines.} We compare our method with four established Self‑Consistency approaches: Self‑Consistency (SC) \citep{wang2023selfconsistencyimproveschainthought}, which samples multiple reasoning paths and selects the final answer via majority voting; Adaptive‑Consistency (AC) \citep{aggarwal2023letssamplestepstep}, which dynamically adjusts the sample count based on answer consensus; and confidence‑based strategies, DeepConf‑Low (D-L)and DeepConf‑High (D-H)  \citep{fu2025deepthinkconfidence}, which employ different confidence thresholds to balance exploration and exploitation.
    
    \textbf{Metrics.} 
    We report accuracy as Acc.@512, computed via confidence-weighted majority voting as described in \cref{Consensus}, and total token consumption in $10^7$ units.
    we repeat 5 times for each dataset and at least 16 times for each query and report the average accuracy.
    
    \textbf{Implementation Details.} We set $T=0.6$, top-$p=0.95$, and sampling budget $B=512$ unless otherwise noted. The init size is $B_{\text{init}}=16$ and $L(2048)$ denoting the sliding window size for group confidence is 2048 tokens. Further implementation details are in Appendix \ref{app:hyperparams}.


    \subsection{Main Results}
    \label{sec:main-results}
    
    \begin{table*}[t] 
    \centering
    \caption{Comprehensive evaluation of different methods across different reasoning benchmarks. We report accuracy (\%) as Acc. and token consumption ($\times 10^{7}$) as Tokens in voting size budget 512 for Majority Voting. Best results are in \textbf{bold}.}
    \label{tab:comprehensive-results}
    \small 
    \setlength{\tabcolsep}{1.8pt}  
    \begin{tabular}{@{}>{\raggedright\arraybackslash}p{1.0cm} *{12}{c} @{}}
        \toprule
        \multirow{2}{*}{\textbf{Method}} &
        \multicolumn{2}{c}{\textbf{MATH500}} &
        \multicolumn{2}{c}{\textbf{AMC23}} &
        \multicolumn{2}{c}{\textbf{AIME24}} &
        \multicolumn{2}{c}{\textbf{AIME25}} &
        \multicolumn{2}{c}{\textbf{GPQA-d}} &
        \multicolumn{2}{c}{\textbf{Avg.}} \\
        \cmidrule(lr){2-3} \cmidrule(lr){4-5} \cmidrule(lr){6-7} \cmidrule(lr){8-9} \cmidrule(lr){10-11} \cmidrule(lr){12-13}
        & Acc.$\uparrow$ & Tokens$\downarrow$ &
        Acc.$\uparrow$ & Tokens$\downarrow$ &
        Acc.$\uparrow$ & Tokens$\downarrow$ &
        Acc.$\uparrow$ & Tokens$\downarrow$ &
        Acc.$\uparrow$ & Tokens$\downarrow$ &
        Acc.$\uparrow$ & Tokens$\downarrow$ \\
        \midrule
        \multicolumn{13}{c}{\cellcolor{gray!15} Qwen3-1.7B} \\
        \midrule
        SC & 76.5 & 44.9 & 94.5 & 18.7 & 74.1 & 49.0 & 49.5 & 47.5 & 40.2 & 186.4 & 67.0 & 69.3 \\
        AC & 76.7 & 24.8 & 94.3 & 10.2 & 74.0 & 19.3 & 49.4 & 17.8 & 40.1 & 69.7 & 66.9 & 28.4 \\
        D-L & 84.0 & 26.2 & 92.5 & 9.7 & \textbf{76.7} & 9.6 & \textbf{53.3} & 9.3 & 40.4 & 126.6 & 69.4 & 36.3 \\
        D-H & 83.8 & 29.6 & 91.1 & 15.7 & \textbf{76.7} & 16.8 & 53.0 & 23.4 & 40.3 & 166.5 & 69.0 & 50.4 \\
        \textbf{DDC} & \textbf{89.4} & \textbf{10.2} & \textbf{95.0} & \textbf{1.5} & 73.3 & \textbf{3.8} & 50.0 & \textbf{3.7} & \textbf{42.9} & \textbf{20.6} & \textbf{70.1} & \textbf{8.0} \\
        \midrule
        \multicolumn{13}{c}{\cellcolor{gray!15} Qwen3-4B} \\
        \midrule
        SC & 83.0 & 48.6 & 96.0 & 16.6 & 80.0 & 38.4 & 66.5 & 41.0 & 48.5 & 80.4 & 74.8 & 45.0 \\
        AC & 83.2 & 18.1 & 96.0 & 4.5 & 79.9 & 12.8 & 66.4 & 15.6 & 48.5 & 23.7 & 74.8 & 14.9 \\
        D-L & 87.0 & 14.1 & \textbf{100.0} & 0.8 & 83.1 & 14.1 & 73.3 & 18.3 & 59.0 & 48.2 & 80.5 & 19.1 \\
        D-H & 87.3 & 30.6 & 96.2 & 10.4 & 80.2 & 24.2 & 66.7 & 25.8 & 48.7 & 50.6 & 75.8 & 28.3 \\
        \textbf{DDC} & \textbf{92.8} & \textbf{7.5} & \textbf{100.0} & \textbf{0.5} & \textbf{83.3} & \textbf{1.2} & \textbf{82.1} & \textbf{1.5} & \textbf{59.1} & \textbf{13.7} & \textbf{83.4} & \textbf{4.9} \\
        \midrule
        \multicolumn{13}{c}{\cellcolor{gray!15}Qwen3-8B} \\
        \midrule
        SC & 85.5 & 51.2 & 96.5 & 15.4 & 80.0 & 23.2 & 78.1 & 27.7 & 63.7 & 74.7 & 80.8 & 38.4 \\
        AC & 85.6 & 13.0 & 96.5 & 6.1 & 80.0 & 12.2 & 70.0 & 9.2 & 63.3 & 24.8 & 79.1 & 13.1 \\
        D-L & 87.9 & 11.4 & 96.7 & 12.5 & \textbf{80.4} & 9.0 & 76.7 & 13.1 & \textbf{65.2} & 33.1 & 81.3 & 15.8 \\
        D-H & 87.8 & 19.3 & 96.0 & 18.3 & 80.1 & 13.3 & \textbf{78.5} & 19.9 & 63.8 & 49.4 & 81.3 & 24.0 \\
        \textbf{DDC} & \textbf{92.8} & \textbf{7.3} & \textbf{97.5} & \textbf{1.0} & 80.0 & \textbf{1.3} & 73.3 & \textbf{2.2} & 63.6 & \textbf{13.7} & \textbf{81.4} & \textbf{5.1} \\
        \midrule
        \multicolumn{13}{c}{\cellcolor{gray!15}Qwen3-32B} \\
        \midrule
        SC & 90.0 & 76.8 & \textbf{97.5} & 10.2 & 84.8 & 20.0 & 80.1 & 24.3 & 71.1 & 74.4 & 84.7 & 41.1 \\
        AC & 90.1 & 21.5 & 97.0 & 3.4 & 84.9 & 10.6 & 80.0 & 9.1 & 70.1 & 23.7 & 84.4 & 13.7 \\
        D-L & 92.7 & 35.8 & 97.3 & 7.9 & 89.4 & 7.8 & 80.2 & 11.4 & 70.9 & 32.1 & 86.1 & 19.0 \\
        D-H & 91.8 & 40.0 & 96.8 & 9.3 & 86.4 & 8.8 & \textbf{80.2} & 16.1 & 70.0 & 41.6 & 85.0 & 23.2 \\
        \textbf{DDC} & \textbf{93.6} & \textbf{6.7} & \textbf{97.5} & \textbf{0.6} & \textbf{93.3} & \textbf{1.3} & 79.7 & \textbf{1.8} & \textbf{72.2} & \textbf{14.1} & \textbf{87.2} & \textbf{4.9} \\
        \midrule
        \multicolumn{13}{c}{\cellcolor{gray!15}DeepSeek-R1-0528-Qwen3-8B} \\
        \midrule
        SC & 91.1 & 53.8 & 93.5 & 14.1 & 86.7 & 35.5 & 82.3 & 40.1 & \textbf{72.5} & 99.2 & 85.2 & 48.5 \\
        AC & 91.0 & 17.9 & 93.5 & 4.7 & 86.6 & 11.8 & 82.3 & 13.3 & 72.4 & 38.0 & 85.2 & 17.1 \\
        D-L & 91.3 & 16.4 & 93.9 & 11.0 & \textbf{92.1} & 7.8 & \textbf{86.4} & 12.4 & 71.7 & 34.6 & \textbf{87.1} & 16.4 \\
        D-H & 91.5 & 25.6 & 94.0 & 9.7 & 86.7 & 14.5 & 81.4 & 23.7 & 72.4 & 69.1 & 85.2 & 28.5 \\
        \textbf{DDC} & \textbf{93.8} & \textbf{9.5} & \textbf{100.0} & \textbf{0.8} & 86.7 & \textbf{1.9} & 83.3 & \textbf{2.4} & 69.6 & \textbf{13.9} & 86.7 & \textbf{5.7} \\
        \bottomrule
    \end{tabular}
    \end{table*}
    
    Table~\ref{tab:comprehensive-results} presents the results on 5 reasoning benchmarks across 5 representative LLMs ranging from 1.7B to 32B.
    
    \textbf{DDC strikes a superior balance between effectiveness and efficiency over strong baselines.} 
    As shown in Table~\ref{tab:comprehensive-results}, DDC achieves superior performance over strong baselines across diverse benchmarks. This advantage is maintained across a wide spectrum of model scales and architectures, ranging from Qwen3-1.7B to 32B and the DeepSeek family. For instance, on the Qwen3-4B backbone, DDC improves the average accuracy by 8.6\% compared to Self-Consistency, and outperforms the strongest baseline, DeepConf-Low, by 2.9\%. Notably, on the challenging AIME25 benchmark, our method achieves a significant accuracy gain of 15.6\% over Self-Consistency. This performance benefits from the dynamic allocation of computational resources, allowing the model to secure higher accuracy on complex reasoning tasks without the exhaustive sampling required by fixed-budget methods.
    
    \textbf{DDC significantly minimizes computational overhead by reducing token consumption.} 
    Beyond accuracy gains, DDC exhibits superior token efficiency. In most scenarios, our method reduces token consumption by approximately $8\times$ compared to the standard Self-Consistency baseline. Crucially, when compared to the state-of-the-art efficiency-oriented method, Adaptive-Consistency, our method further reduces token usage by over $3\times$
     on average while maintaining higher accuracy. For example, on the AIME25 benchmark with Qwen3-4B, DDC consumes $27\times$ fewer tokens than the voting baseline. This indicates that existing adaptive methods may not optimally evaluate the halting criteria, often generating redundant tokens, whereas DDC effectively terminates generation once a confident answer is derived. 
     While DDC introduces additional analytical components like trend analysis, the associated computational overhead is negligible compared to the LLM inference process. Conversely, the significant reduction in token consumption translates directly into a massive reduction. DDC reduces total latency by up to $12.4\times$ relative to the fastest baseline. Detailed complexity analysis and detailed generation time is provided in Appendix~\ref{exp}.
    
    \subsection{On the Inference-Time scaling}
    \label{sec:Inference-Time scaling-analysis}
    Figure~\ref{inference_time_scaling} presents the inference-time scaling law on Qwen3-4B. From the scaling curves, DDC presents the consistency superiority on each token consumption, ranging from $0.89 \times 10^{7}$ to $20 \times 10^{7}$. Furthermore, DDC demonstrates over $10 \times$ efficiency than suboptimal methods while also delivering superior performance. Detailed results is in the Appendix~\ref{exp}.
    
    \subsection{Superiority over compute-intensive methods.}
    Table~\ref{tab:compute-intensive} demonstrates that DDC surpasses computationally exhaustive strategies like Predictive Decoding \citep{ma2024nonmyopicgenerationlanguagemodels} and $\phi$-Decoding \citep{xu2025phidecodingadaptiveforesightsampling}. On AIME25 with Qwen3-4B, our method exceeds $\phi$-Decoding by 15.0\% while reducing token consumption by $29\times$. While $\phi$-Decoding holds a marginal +3.0\% advantage on GPQA-diamond with DeepSeek-R1-0528-Qwen3-8B, it incurs a prohibitive $6.9\times$ computational overhead. This confirms that rigorous quality control superior reasoning results compared to unguided brute-force scaling.
    
        \begin{table*}[!h]
    	\centering
    	\begin{minipage}{0.48\textwidth}
    		\centering
    		\caption{Comparison with compute-intensive methods.}
    		\label{tab:compute-intensive}
    		\resizebox{0.91\linewidth}{!}{%
    			\begin{tabular}{lcccc}
    				\toprule
    				\multirow{2}{*}{Method} & \multicolumn{2}{c}{GPQA-diamond} & \multicolumn{2}{c}{AIME25} \\
    				\cmidrule(lr){2-3} \cmidrule(lr){4-5}
    				& Acc. & Tokens & Acc. & Tokens \\
    				\midrule
    				\multicolumn{5}{c}{\cellcolor{gray!15} Qwen3-4B} \\
    				\midrule
    				\textbf{DDC} & \textbf{59.1} & \textbf{13.7} & \textbf{82.1} & \textbf{1.5} \\
    				Predictive Decoding + SC & 54.8 & 94.3 & 66.8 & 48.9 \\
    				$\phi$-Decoding + SC & 56.6 & 86.9 & 67.1 & 43.3 \\
    				\midrule
    				\multicolumn{5}{c}{\cellcolor{gray!15} DeepSeek-R1-0528-Qwen3-8B} \\
    				\midrule
    				\textbf{DDC} & 69.6 & \textbf{13.9} & \textbf{83.3} & \textbf{2.4} \\
    				Predictive Decoding + SC & 69.5 & 98.6 & 78.0 & 57.1 \\
    				$\phi$-Decoding + SC & \textbf{72.6} & 95.8 & 79.2 & 52.6 \\
    				\bottomrule
    			\end{tabular}%
    		}
    	\end{minipage}%
    	\hfill
    	\begin{minipage}{0.48\textwidth}
    		\centering
    		\caption{Ablation Studies on Qwen3-4B and DeepSeek-R1-0528-Qwen3-8B.}
    		\label{tab:ablation}
    		\resizebox{0.92\linewidth}{!}{%
    			\begin{tabular}{lcccc}
    				\toprule
    				\multirow{2}{*}{Method} & \multicolumn{2}{c}{MATH500} & \multicolumn{2}{c}{AIME25} \\
    				\cmidrule(lr){2-3} \cmidrule(lr){4-5}
    				& Acc. & Tokens & Acc. & Tokens \\
    				\midrule
    				\multicolumn{5}{c}{\cellcolor{gray!15}Qwen3-4B} \\
    				\midrule
    				\textbf{DDC} & \textbf{92.4} & \textbf{10.2} & \textbf{82.1} & \textbf{1.5} \\
    				w/o Weighted Bayesian & 90.2 & 25.6 & 80.0 & 6.1 \\
    				w/o Inter-path Scaling & 88.1 & 29.7 & 77.6 & 21.6 \\
    				w/o Trend-Aware & 92.0 & 7.7 & 75.3 & 1.1 \\
    				w/o Intra-path Pruning & 85.5 & 18.4 & 70.9 & 10.3 \\
    				\midrule
    				\multicolumn{5}{c}{\cellcolor{gray!15}DeepSeek-R1-0528-Qwen3-8B} \\
    				\midrule
    				\textbf{DDC} & \textbf{93.8} & \textbf{9.5} & \textbf{83.3} & \textbf{2.4} \\
    				w/o Weighted Bayesian & 91.9 & 10.1 & 82.6 & 6.6 \\
    				w/o Inter-path Scaling & 92.3 & 20.9 & 81.6 & 19.3 \\
    				w/o Trend-Aware & 93.6 & 9.0 & 83.1 & 2.0 \\
    				w/o Intra-path Pruning & 91.2 & 18.4 & 82.5 & 14.9 \\
    				\bottomrule
    			\end{tabular}%
    		}
    	\end{minipage}
    \end{table*}    

    \begin{figure*}[ht]
      \centering
      \includegraphics[width=0.8\textwidth]{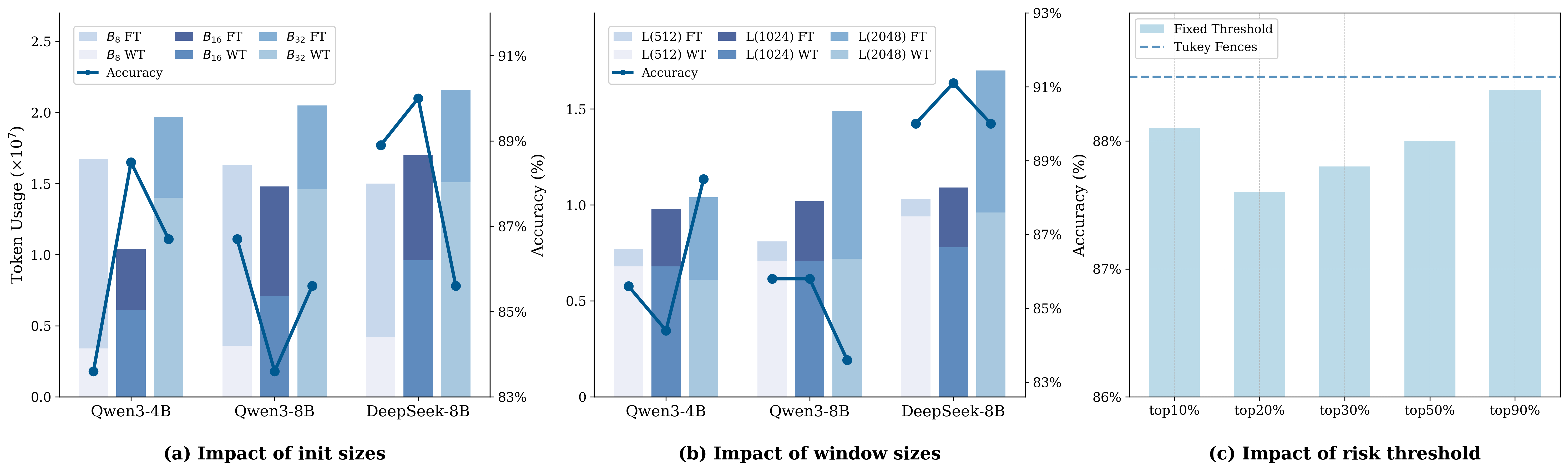}
      \caption{
        Hyperparameter sensitivity analysis.
        We report average accuracy and token consumption on competition benchmarks:
        (a) Impact of init sizes. 
        (b) Impact of window sizes.
         (c) Comparison between adaptive Tukey Fences and static percentage-based strategies.
      }
      \label{fig:warmup}
    \end{figure*}
    
    \subsection{Ablation Studies}

    Table~\ref{tab:ablation} reveals that DDC's performance stems from a functional interdependency between Inter-path and Intra-path.

    \textbf{Weighted Bayesian accelerates convergence.} 
    Reverting to standard frequency-based termination (\textit{w/o Weighted Bayesian}) increases token consumption by $\sim 2.5\times$ on MATH500 while dropping accuracy by over 2\%. This validates that weighting consensus by path quality prevents stops on low confidence answers and accelerates convergence to high quality answers. Inter-path module requires the Intra-path quality signal and the quality derived from Inter-path termination, serves as the critical coupling bridge that grounds the Intra-path trend analysis.
    
    \textbf{Inter-path Scaling optimizes resource allocation.} 
    Disabling adaptive budgeting (\textit{w/o Inter-path Scaling}) causes a prohibitive $>2.9\times$ token spike and an accuracy drop of up to 4.5\%. This indicates that our method safeguards accuracy by locking in the result immediately upon high quality convergence, avoiding diluting the correct consensus with incorrect ones. 
    
    \textbf{Trend-aware analysis mitigates hallucinations.} 
    Replacing dynamic signal analysis with static confidence pruning (\textit{w/o Trend-Aware}) leads to a 6.8\% accuracy decline on AIME25. This demonstrates that unlike static thresholds, our trend-aware mechanism successfully distinguishes the temporary confidence dips typical of deep reasoning from hallucinations.
    
    \textbf{Intra-path Pruning prevents error propagation.} 
    Removing depth-wise verification entirely (\textit{w/o Intra-path Pruning}) results in the largest performance degradation on MATH500 from 92.4\% to 85.5\%. This proves that aggressive early truncation is critical for filtering  incorrect paths before they pollute the final majority vote.
    
    \textbf{Hyperparameter Analysis.} 
    We analyze init sizes $B_{\text{init}}$ and window granularity $L$ across three competition datasets in Figure~\ref{fig:warmup}. Results indicate that moderate settings $B_{\text{init}}=16$, $L=2048$ achieve the best trade-off between accuracy and token efficiency. While models such as Qwen3‑8B and DeepSeek-R1-0528-Qwen3-8B underperform in \cref{tab:comprehensive-results}, adjusting $B_{\text{init}}$ and $L$ can improve accuracy.  Our adaptive $\tau_{\text{risk}}$ thresholding consistently outperforms static percentage-based strategies. Although sensitivity exists, DDC exhibits a broad, stable performance plateau for \( B_{\text{init}} \) and \( L \) across diverse models and datasets. Notably, all results in \cref{tab:comprehensive-results} were achieved using a unified hyperparameter set for all models from 1.7B to 32B, demonstrating strong generalizability. The calibration in DDC is an automatic, per-query statistical derivation via Tukey Fences rather than a manual tuning effort, requiring zero computational overhead for new datasets or models.
    Detailed experiments are provided in Appendix~\ref{exp}.
	
	\section{Conclusion}
    This work addresses the inefficiency of static and disjoint adaptive inference scaling by introducing DDC, a unified framework that balances budget and quality. By coupling a Confidence-Weighted Bayesian termination protocol with Trend-Aware Stratified Pruning, our method dynamically concentrates computational resources on high quality reasoning chains while identifying and discarding hallucinations. Extensive evaluations demonstrate that this approach achieves state-of-the-art performance, offering a scalable and highly efficient pathway for deploying reliable reasoning in resource-constrained environments.
	
	%
	%

	\nocite{langley00}



	\bibliography{ref}
	\bibliographystyle{plainnat}

	\newpage
	\appendix
	\onecolumn
    \section{Implementation Details}
    \label{app:implementation}
    For all experiments, we set the temperature to 0.6 and top-$p$ to 0.95 unless otherwise specified. 
    
    \subsection{Datasets Description}
    \label{app:datasets-models}
    
    We evaluate on five challenging reasoning benchmarks:
    
    \begin{itemize}
        \item \textbf{MATH-500}: A subset of 500 problems from the MATH dataset  \citep{hendrycks2021measuringmathematicalproblemsolving} covering algebra, geometry, calculus, and number theory. Problems require multi-step reasoning and symbolic manipulation.
        
        \item \textbf{AMC23}: The 2023 American Mathematics Competition problems, consisting of 30 challenging mathematical problems designed for high school students.
        
        \item \textbf{AIME24}: The 2024 American Invitational Mathematics Examination (AIME) problems, comprising 30 advanced problems that require deep mathematical insight and creative problem-solving.
        
        \item \textbf{AIME25}: The 2025 AIME problems, featuring 15 problems that test advanced mathematical reasoning and proof techniques.
        
        \item \textbf{GPQA-diamond}: A high-quality subset of the Graduate-Level Google-Proof Q\&A (GPQA) dataset  \citep{rein2023gpqagraduatelevelgoogleproofqa}, containing 198 questions across physics, chemistry, and biology that require expert-level knowledge and reasoning.
    \end{itemize}
    
    All datasets are in English. For each question, the model generates step-by-step reasoning before producing the final answer.
    
    \subsection{Metrics Calculation}
    We report two main metrics:
    
    \paragraph{Accuracy (Acc.@K)}: For each question, the max budget is $K=512$ reasoning paths. The final answer is selected using confidence-weighted majority voting as described in Section \ref{Consensus}. Accuracy is computed as the percentage of questions where the selected answer matches the ground truth.
    
    \paragraph{Token Consumption}: We measure the total number of tokens generated  across all paths for all questions in the dataset. This includes tokens from both the init phase and the main scaling phase. Token counts are reported in units of $10^7$ for readability.
    
    \paragraph{Additional Metrics}: For the inference-time scaling analysis in Figure \ref{inference_time_scaling}, we compute the average accuracy across all five benchmarks at various token budget levels, ranging from $0.89 \times 10^7$ to $20 \times 10^7$ tokens.
    
    \subsection{Hyperparameter Settings}
    \label{app:hyperparams}
    
    The key hyperparameters for Dual-Dimensional Consistency are:
    
    \begin{itemize}
        \item \textbf{init size ($B_{\text{init}}$)}: 16 samples for all experiments unless otherwise specified.
        \item \textbf{Sampling budget ($B$)}: 512 samples for all main experiments.
        \item \textbf{Sliding window size ($L$)}: 2048 tokens for group confidence calculation.
        \item \textbf{Majority threshold ($\gamma$)}: 0.5 denoting absolute majority.
        \item \textbf{Stopping confidence threshold ($\tau_{\text{stop}}$)}: 0.95.
        \item \textbf{Trend analysis window size}: 2048 tokens.
        \item \textbf{Penalty coefficient ($\eta$)}: 0.5 in the Structural Instability Score $\mathcal{R}$.
    \end{itemize}

    \section{Theoretical Analysis and Proofs}
    \label{sec:theory}
    
    In this section, we provide a formal analysis of the Dual-Dimensional Consistency (DDC) framework, grounding our approach in established statistical literature.
    
    \subsection{Inter-Path: Generalized Bayesian Termination}
    
    \subsubsection{Formulation as Generalized Posterior}
    To integrate continuous confidence weights $w_n = C_{\text{path}}(y^{(n)}) \in [0, 1]$ into the consensus decision, we adopt the Power Prior framework   \citep{2000power}. This framework provides a principled method for incorporating historical data or auxiliary information through power parameters. 
    
    Consider the Bernoulli likelihood for a single observation $u_n \in \{0,1\}$, where $u_n=1$ indicates a correct answer. The weighted log-likelihood is:
    \begin{equation}
        \log L_w(\mathcal{D}_N \mid \theta) = \sum_{n=1}^N w_n \left[ u_n \log \theta + (1-u_n) \log (1-\theta) \right].
    \end{equation}
    
    This formulation corresponds to a generalized Bayesian update where each observation is weighted by its confidence, analogous to importance sampling in Monte Carlo methods   \citep{Hesterberg2002Monte}. Combined with a conjugate Beta prior $\text{Beta}(\alpha_0, \beta_0)$, this yields a posterior:
    \begin{equation}
        \theta \mid \mathcal{D}_N \sim \text{Beta}\left(\alpha_0 + \sum_{n=1}^N w_n u_n, \beta_0 + \sum_{n=1}^N w_n (1-u_n)\right).
    \end{equation}
    
    Crucially, the posterior variance sales as $\text{Var}(\theta \mid \mathcal{D}_N) \propto (\sum w_n)^{-1}$, following standard Beta distribution properties   \citep{Gelman2014Bayesian}. Thus, low-confidence weights naturally inflate uncertainty. This mechanism acts as a statistical safety brake: if weights are low, the model requires a larger number of raw samples to reach the same posterior confidence threshold compared to the unweighted case.
    
    \subsubsection{Efficiency Analysis via Sample Complexity}
    
    We analyze the conditions under which Confidence-Weighted Bayesian Termination (CoW) achieves faster convergence than standard Frequency-Based Bayesian Termination (FrQ) using techniques from sequential analysis   \citep{1945Sequential}. Let $p = \Pr(u_n = 1)$ be the model's accuracy, and $q = 1-p$.
    
    \begin{figure*}[ht]
      \centering
      \includegraphics[width=\textwidth, height=\textheight, keepaspectratio]{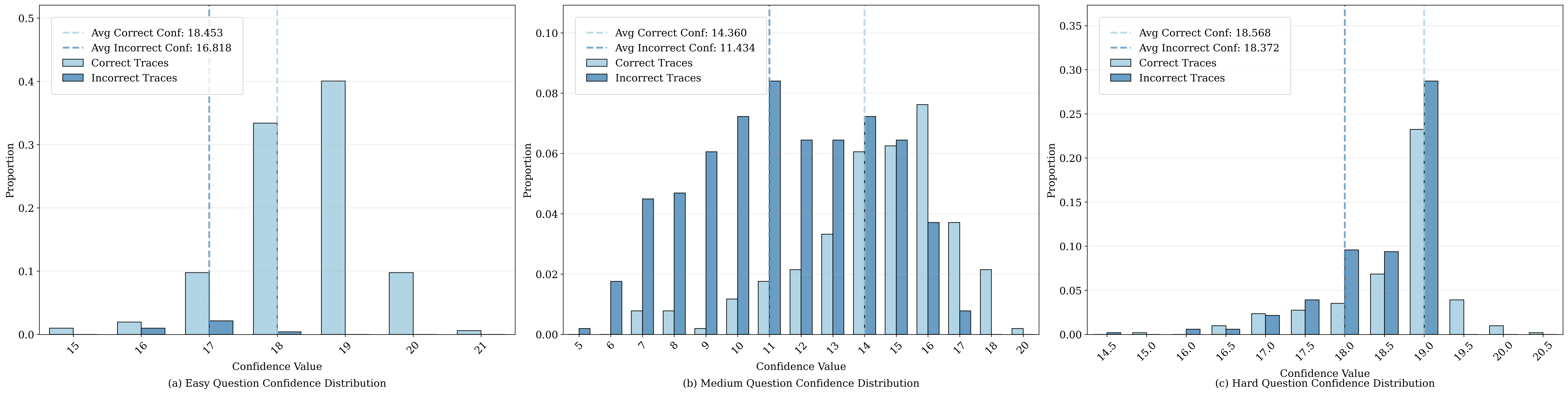}
      \caption{
        Analyzing the confidence distribution for problems of varying complexity.
      }
      \label{fig:confidence_distribution}
    \end{figure*}
    \begin{theorem}[Sufficient Condition for Acceleration]
    \label{thm:adaptive_efficiency}

    Let $N_{FrQ}$ and $N_{CoW}$ be the stopping times for Self-Consistency and CoW to reach a fixed confidence threshold $Z_\tau$. Assume:
    
    \textbf{Assumption 1 (Solvability):} $p > 0.5$.
    \textbf{Assumption 2 (Informative Weights):} Confidence weights are informative, meaning correct paths receive higher weights on average: $\mu_c > \mu_i$, where $\mu_c = \mathbb{E}[w \mid y=1]$ and $\mu_i = \mathbb{E}[w \mid y=0]$ as indicated in \citep{fu2025deepthinkconfidence}. Our experimental results of different complexity of MATH-500 \citep{hendrycks2021measuringmathematicalproblemsolving} in \cref{fig:confidence_distribution} demonstrates this, showing a higher average confidence for correct paths.
    
    Define the second moments $\mu_c^{(2)} = \mathbb{E}[w^2 \mid y=1]$ and $\mu_i^{(2)} = \mathbb{E}[w^2 \mid y=0]$.
    Then, CoW is strictly more efficient with $N_{CoW} < N_{FrQ}$ in the large sample limit if:
    \begin{equation}
        \label{eq:sufficient_condition}
        \frac{p\mu_c - q\mu_i}{\sqrt{p\mu_c^{(2)} + q\mu_i^{(2)}}} > \frac{p-q}{\sqrt{4pq}}.
    \end{equation}
    \end{theorem}
    
    \begin{proof}
    Termination requires the Z-sore of evidence to exceed $Z_\tau$, following standard sequential testing theory. For FrQ, the Z-sore is derived from the binomial proportion test \citep{Wickens1998Categorical}:
    \[
    Z_{FrQ}(N) = \sqrt{N} \frac{p-q}{\sqrt{4pq}} + o(1).
    \]
    
    For CoW, we analyze the weighted process. Let $X_n = w_n(2u_n - 1)$ be the signed weight at step $n$. Then the expected weighted drift is $\mu_{CoW} = \mathbb{E}[X_n] = p\mu_c - q\mu_i$, and the variance is $\sigma^2_{CoW} = \mathbb{E}[X_n^2] - \mu_{CoW}^2$. Using the inequality $\sigma^2_{CoW} \le \mathbb{E}[X_n^2] = p\mu_c^{(2)} + q\mu_i^{(2)}$, we construct a conservative lower bound for the CoW Z-sore following the methodology of \citep{2014Sequential}:
    \[
    \tilde{Z}_{CoW}(N) = \sqrt{N} \frac{p\mu_c - q\mu_i}{\sqrt{p\mu_c^{(2)} + q\mu_i^{(2)}}}.
    \]
    
    Since $\sigma^2_{CoW} \le \mathbb{E}[X_n^2]$, the true Z-sore satisfies $Z_{CoW}(N) \ge \tilde{Z}_{CoW}(N)$ by application of Chebyshev's inequality \citep{1969Random}. Let $N_{FrQ}$ be the smallest $N$ such that $Z_{FrQ}(N) \ge Z_\tau$. Condition \eqref{eq:sufficient_condition} implies that for $N = N_{FrQ}$:
    \[
    \tilde{Z}_{CoW}(N_{FrQ}) > Z_{FrQ}(N_{FrQ}) \ge Z_\tau.
    \]
    
    Since $Z_{CoW}(N_{FrQ}) \ge \tilde{Z}_{CoW}(N_{FrQ})$, we have $Z_{CoW}(N_{FrQ}) > Z_\tau$. Therefore, CoW would have terminated by $N_{FrQ}$ or earlier, establishing $N_{CoW} \le N_{FrQ}$.
    \end{proof}
    
    \textbf{Safety Brake Mechanism:} When weights are uniformly small $w_n \approx \epsilon \ll 1$, both sides of inequality \eqref{eq:sufficient_condition} sale with $\epsilon$, preserving the condition. However, the effective sample size in the Bayesian posterior grows as $\epsilon N$, as shown in the effective sample size analysis of weighted importance sampling \citep{1964Monte}. Thus, even when the signal-to-noise ratio condition holds, CoW requires approximately $1/\epsilon$ times more raw samples than FrQ when confidence is low. This conservatism provides robustness against premature conclusions when the model is uncertain.

    \subsubsection{Empirical Verification}
    \begin{wrapfigure}[13]{r}{0.5\textwidth}
        \centering
        \vspace{-1em}
       \includegraphics[width=1\linewidth]{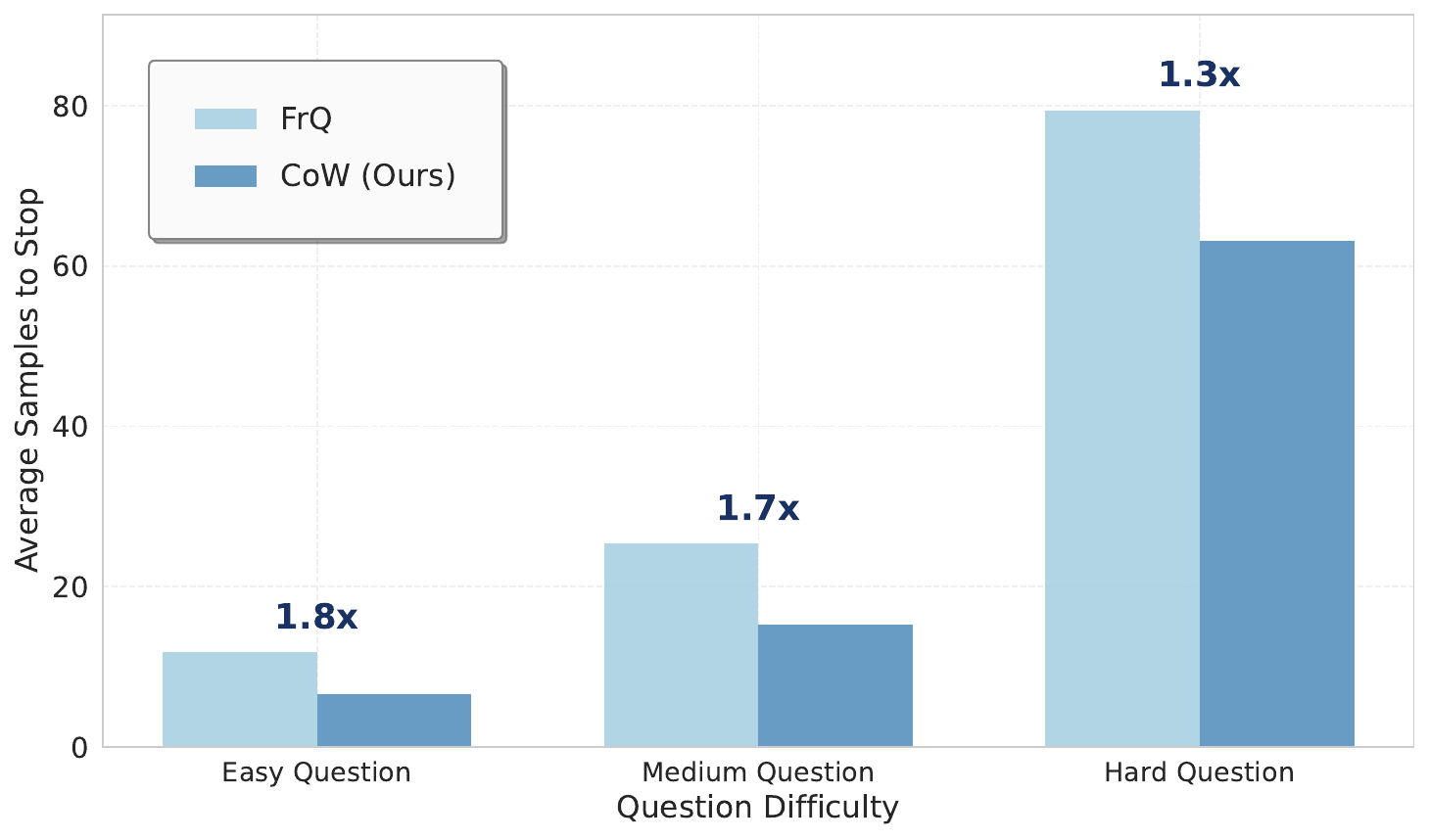}  
       \vspace{-1.5em}
        \caption{Empirical verification.  }
        \label{fig:efficiency_comparison}
    \end{wrapfigure} 
    To empirically validate Theorem \ref{thm:adaptive_efficiency}, we conducted Monte Carlo simulations across varying question difficulties, where Easy Questions denotes level 1 selected from MATH500 \citep{hendrycks2021measuringmathematicalproblemsolving}, Medium Questions denotes level 3 and Hard Questions denotes level 5. 
    
    As shown in Figure \ref{fig:efficiency_comparison}, CoW demonstrates significant efficiency gains over FrQ when confidence information is reliable. 
    These empirical results confirm the theoretical predictions of Theorem \ref{thm:adaptive_efficiency}. When confidence weights are informative, CoW achieves faster convergence by leveraging the additional signal contained in confidence estimates. Moreover, the simulations verify that our method naturally adapts to question difficulty providing efficiency gains when possible while maintaining reliability when confidence signals are ambiguous.
    
    
    \subsection{Intra-Path: Trend-Aware Pruning Analysis}
    
    \subsubsection{Statistical Foundation of Pruning Threshold}
    We employ Tukey Fences   \citep{2008Exploratory}: $\tau_{\text{risk}} = Q_3 + 1.5 \cdot \text{IQR}$, applied to the instability sore $\mathcal{R}$. This nonparametric threshold is robust to outliers and non-Gaussian distributions, addressing the skewed confidence distributions commonly observed in LLM reasoning. The interquartile range (IQR) provides a robust measure of dispersion, making the threshold adaptive to query difficulty without parametric assumptions.
    
    \subsubsection{Analysis of Failure Modes}
    Let $\mathbf{Z} \in \mathbb{R}^{k \times 2}$ be the phase-space trajectory with covariance eigenvalues $\lambda_1 \ge \lambda_2$. Spectral analysis of confidence trajectories follows established methods for time series analysis   \citep{2006Time}.
    
    While spectral analysis effectively filters stochastic noise, we must address hallucinations where the model is confident in a wrong answer. To guarantee correctness in this regime, we introduce the following assumption regarding the ensemble distribution:
    
    \textbf{Assumption 3 (Inter-Path Diversity):} We assume that error modes are not perfectly correlated across the sampling distribution. Formally, the expected weighted support for the correct answer $u^*$ exceeds that of any single incorrect answer $u \neq u^*$:
    \begin{equation}
        \mathbb{E}\left[\sum_{n} w_n \mathbb{I}(y_n = u^*)\right] > \max_{u \neq u^*} \mathbb{E}\left[\sum_{n} w_n \mathbb{I}(y_n = u)\right].
    \end{equation}
    
    This assumption is empirically supported by findings that LLM errors are often diverse and uncorrelated across reasoning paths   \citep{wang2023selfconsistencyimproveschainthought}. Under this assumption, the Confidence-Weighted Bayesian consensus successfully suppresses individual high-confidence errors by aggregating evidence from diverse correct paths, following the principle of the wisdom of crowds   \citep{2005The}.
    
    \begin{proposition}[Pruning Efficacy]
    \label{prop:pruning_efficacy}
    Under Assumption 3 and the Tukey Fences threshold $\tau_{\text{risk}}$:
    \begin{enumerate}
        \item \textbf{Stochastic Hallucinations} are pruned with high probability due to their maximal instability sores.
        \item \textbf{Decaying Hallucinations} are pruned via the velocity penalty term in $\mathcal{R}$.
        \item \textbf{Self-Reinforcing Hallucinations} pass the intra-path filter but are suppressed in the inter-path consensus under Assumption 3.
    \end{enumerate}
    \end{proposition}
    \begin{proof}
    The first two claims follow directly from the construction of $\mathcal{R}$ and the properties of Tukey Fences for outlier detection   \citep{Technometrics0Understanding}. The third claim follows from Assumption 3 and the consistency of weighted Bayesian aggregation   \citep{Jos1994Bayesian}.
    \end{proof}
    
    \subsection{Limitations and Broader Connections}
    \label{limitations}
    \begin{table}
      \caption{Computational cost comparison on MATH500 and AIME25. Width refers to the average number of generated paths per question; Token denotes the average token consumption for high-confidence but incorrect predictions. }
      \label{tab:compute-compare}
      \centering
      \begin{tabular}{lcccc}
        \toprule
        \multirow{2}{*}{Method} & \multicolumn{2}{c}{MATH500} & \multicolumn{2}{c}{AIME25} \\
        \cmidrule(lr){2-3} \cmidrule(lr){4-5}
        & Width & Token & Width & Token \\
        \midrule
        \multicolumn{5}{c}{\cellcolor{gray!15} Qwen3-4B} \\
        \midrule
        SC & 512 & 0.81 & 512 & 0.96 \\
        \textbf{DDC} & \textbf{62} & \textbf{0.17} & \textbf{94} & \textbf{0.19} \\
        \bottomrule
      \end{tabular}
    \end{table}
    \textbf{Limitations and Future Work:} DDC inherently relies on the strong correlation between confidence and correctness, offering no mechanism to circumvent systematic miscalibration \citep{2017On}, where in the model is confidently wrong across all reasoning paths.
    
    The Bayesian termination in DDC is fundamentally driven by the contrast between the evidence strength of the dominant candidate, 
    \begin{equation}
    \alpha_u = 1 + \sum w_i \cdot I\{\mathrm{ans}_i = u\},
    \end{equation}
    and the aggregate counter-evidence,
    \begin{equation}
    \beta_u = 1 + \sum w_j \cdot I\{\mathrm{ans}_j \neq u\}.
    \end{equation}
    In scenarios where a model exhibits systematic miscalibration, it generates reasoning paths that are both identically incorrect and highly confident. Consequently, $\alpha_{u'}$ grows with each new sample, reflecting an accumulation of high-quality evidence in the eyes of the base model, while $\beta_{u'}$ remains near its initial Laplace prior value, as no dissenting paths are generated to challenge the false consensus. This causes the posterior distribution $P(p_{u'} > 0.5 \mid D)$ to rapidly converge to the reliability threshold $\tau_{\mathrm{stop}} = 0.95$, typically triggering an early exit within just a few paths.
    
    Empirically, while the model predicts correctly with high confidence in the vast majority of scenarios \cref{fig:confidence_distribution}, we observe and analyze a minority of confidently incorrect cases in \cref{tab:compute-compare}. We acknowledge that such premature termination on an incorrect answer is an inherent boundary of any consensus-based strategy: if a model persistently adheres to a fallacy with high subjective certainty, no sampling-based method can rectify it via internal signals alone. However, by identifying that additional sampling would only reinforce the same incorrect consensus, DDC effectively reallocates the computational budget from unresolvable hallucinations to complex but solvable queries, acting as an efficient scaling mechanism that respects the model's current epistemic limits.
    
    To further directly address this boundary, a promising avenue for future work is to integrate advanced confidence calibration techniques into the DDC framework. By incorporating post-hoc calibration or intrinsic epistemic uncertainty estimation to dynamically penalty or re-weight the evidence scores $w_i$, we anticipate that the framework can better decouple true consensus from systematic hallucinations, thereby further expanding the safe operating boundaries of inference-time scaling.

    \textbf{Connections:} DDC unifies \textit{Importance Sampling}   \citep{Hesterberg2002Monte} for reweighting evidence with \textit{Sequential Probability Ratio Testing}   \citep{1945Sequential} for adaptive stopping, situated within the Power Prior Bayesian framework   \citep{Joseph2015The}.

    \subsection{Impact Statement}
    \label{impact}
    This paper advances efficient inference for large language models by unifying the optimization of sampling budget and reasoning path quality. The potential societal impact includes enabling more efficient and robust approaches for reliable reasoning. To the best of our knowledge, this work does not pose notable ethical or societal risks. Our goal is to support the responsible development of scalable and efficient reasoning methods.

    \section{More Experiment Results}
    \label{exp}
    
    \subsection{Theoretical Complexity Analysis of DDC Operations}
    To address concerns regarding the computational overhead introduced by Trend-Aware Stratified Pruning and Bayesian Termination, we provide a formal complexity analysis comparing DDC's auxiliary operations with the standard LLM inference process.

    \textbf{Dimensionality of Auxiliary Operations:}
        For each generation step $t$, DDC maintains a state vector $z_t \in \mathbb{R}^d$ where $d = 2$ representing the normalized position and velocity of global confidence. Over a sliding window of size $k$, we construct a trajectory matrix $Z \in \mathbb{R}^{k \times 2}$.
        
        \begin{itemize}
            \item \textit{Covariance Matrix Construction:} Computing $\Sigma = \frac{1}{k - 1} Z^{\top} Z$ involves a matrix multiplication of $2 \times k$ and $k \times 2$. The computational cost is $O(k \cdot d^2)$. Since $d = 2$, this is effectively $4k$ operations.
            
            \item \textit{Eigendecomposition:} Performing eigendecomposition on a $2 \times 2$ matrix is an analytically solvable problem with a constant complexity of $O(d^3) = O(8)$.
            
            \item \textit{Bayesian Update:} The calculation of the incomplete beta function $I_{\gamma}(\alpha, \beta)$ is a scalar operation performed once per path completion, with negligible cost.
        \end{itemize}
        
    \textbf{Quantifying the Overhead Ratio:}
        For a standard LLM with $N$ parameters like 7B or 32B, the computation required for generating a single token is approximately $2N$ FLOPs.
        
        For Qwen3-32B, one forward pass requires $\approx 64 \times 10^9$ FLOPs per token.
        
        The DDC analysis for $k = 2048$ requires $\approx 8 \times 10^3$ FLOPs per token.
        
        Overhead Ratio: 
        \[
        \frac{\text{FLOPs}_{\text{DDC}}}{\text{FLOPs}_{\text{LLM}}} \approx 1.25 \times 10^{-7}.
        \]
        
        This confirms that the analytical overhead of DDC is seven orders of magnitude smaller than the core inference cost. The real world bottleneck remains the LLM's forward pass, meaning any reduction in token consumption translates almost 1:1 into time savings.
    
    \begin{table*}[!t]
    \centering
    \caption{Latency Comparison between DeepConf-Low and DDC. We report the Accuracy as Acc., per-question warmup generation time as WG, final generation time as FG, and total generation time as TG in unit of seconds. All values are averaged over all questions. For accuracy, higher is better; for time, lower is better. The best results are in bold.}
    \footnotesize
    \label{tab:comparison-time-acc}
    \setlength{\tabcolsep}{1.5pt}  
    \begin{tabular}{@{}>{\raggedright\arraybackslash}p{1.2cm}*{12}{c}@{}}  
        \toprule
        \multirow{1}{*}{\textbf{Method}} & 
        \multicolumn{4}{c}{\textbf{AMC23}} & 
        \multicolumn{4}{c}{\textbf{AIME24}} & 
        \multicolumn{4}{c}{\textbf{AIME25}} \\
        \cmidrule(lr){2-5} \cmidrule(lr){6-9} \cmidrule(lr){10-13}
        & Acc. $\uparrow$ & WG $\downarrow$ & FG $\downarrow$ & TG $\downarrow$ & 
        Acc. $\uparrow$ & WG $\downarrow$ & FG $\downarrow$ & TG $\downarrow$ & 
        Acc. $\uparrow$ & WG $\downarrow$ & FG $\downarrow$ & TG $\downarrow$ \\
        
        \midrule
        \multicolumn{13}{c}{\cellcolor{gray!15}Qwen3-0.6B} \\
        \midrule
        D-L & 65.0 & \textbf{90.92} & 1098.35 & 1189.26 & 33.0 & \textbf{116.20} & 1571.14 & 1687.34 & 30.0 & \textbf{116.20} & 1571.14 & 1687.34 \\
        \textbf{DDC} & \textbf{70.0} & 212.03 & \textbf{328.43} & \textbf{540.46} & \textbf{36.7} & 386.31 & \textbf{618.89} & \textbf{1005.20} & \textbf{36.7} & 305.59 & \textbf{538.11} & \textbf{843.71} \\
        \midrule
        \multicolumn{13}{c}{\cellcolor{gray!15}Qwen3-1.7B} \\
        \midrule
        D-L & 92.5 & \textbf{97.05} & 1470.12 & 1567.17 & \textbf{76.7} & \textbf{131.13} & 2141.64 & 2272.76 & \textbf{53.3} & \textbf{127.83} & 2042.16 & 2169.99 \\
        \textbf{DDC} & \textbf{95.0} & 447.05 & \textbf{108.76} & \textbf{555.81} & 73.3 & 748.99 & \textbf{531.82} & \textbf{1280.81} & 50.0 & 808.66 & \textbf{487.71} & \textbf{1296.37} \\
        \midrule
        \multicolumn{13}{c}{\cellcolor{gray!15}Qwen3-4B} \\
        \midrule
        D-L & 100.0 & \textbf{494.25} & 26.65 & \textbf{520.89} & 83.1 & 880.10 & 7721.52 & 8601.62 & 73.3 & 1093.82 & 13483.31 & 14577.12 \\
        \textbf{DDC} & \textbf{100.0} & 506.28 & \textbf{17.62} & \textbf{523.90} & \textbf{83.3} & \textbf{811.43} & \textbf{71.00} & \textbf{882.43} & \textbf{82.1} & \textbf{857.79} & \textbf{149.24} & \textbf{1007.03} \\
        \midrule
        \multicolumn{13}{c}{\cellcolor{gray!15}Qwen3-8B} \\
        \midrule
        D-L & 96.7 & \textbf{469.15} & 5058.45 & 5527.59 & \textbf{80.4} & 1056.11 & 10610.01 & 11666.12 & \textbf{76.7} & 1374.06 & 15233.55 & 16607.62 \\
        \textbf{DDC} & \textbf{97.5} & 492.25 & \textbf{56.19} & \textbf{548.45} & 80.0 & \textbf{833.45} & \textbf{130.89} & \textbf{964.35} & 73.3 & \textbf{940.04} & \textbf{226.82} & \textbf{1166.86} \\
        \midrule
        \multicolumn{13}{c}{\cellcolor{gray!15}Qwen3-32B} \\
        \midrule
        D-L & 97.3 & 708.83 & 7610.27 & 8315.10 & 89.4 & 1584.69 & 15915.42 & 17500.11 & \textbf{80.2} & 2061.38 & 22850.73 & 24912.11 \\
        \textbf{DDC} & \textbf{97.5} & \textbf{706.61} & \textbf{36.89} & \textbf{743.51} & \textbf{93.3} & \textbf{1296.13} & \textbf{197.42} & \textbf{1493.55} & 79.7 & \textbf{1617.74} & \textbf{393.36} & \textbf{2011.10} \\
        
        \bottomrule
    \end{tabular}
    \end{table*}

    \subsection{Analysis of Inference-Time Efficiency and Latency}
    As presented in Table \ref{tab:comparison-time-acc}, we report the time consumption across different benchmarks. The time is decomposed into warmup generation in initial $B_{init}$ profiling and final generation in adaptive scaling phase, and total generation. Compared to the fastest baseline, DDC significantly reduces the total generation time across all model sizes.
    
    \textbf{Superior Efficiency on Large-Scale Models:}
        The advantage of DDC becomes more pronounced as the model size increases. For Qwen3-32B on the AIME25 benchmark: DeepConf-Low requires 24,912 seconds of total generation time, but DDC requires only 2,011 seconds, achieving a $12.4\times$ speedup.
        This dramatic reduction is primarily due to DDC's ability to trigger early termination much sooner than DeepConf-Low while maintaining comparable accuracy. For instance, for Qwen3-32B on the AIME25, DDC's FG time is 393s, whereas DeepConf-Low's is 22850s, a $58\times$ faster finalization.
    
    \textbf{The Warmup-to-Final Trade-off:}
        In some scenarios like Qwen3-1.7B, DDC incurs a higher WG time compared to DeepConf-Low. This is because DDC's initialization phase involves more thorough confidence profiling to establish a robust query-specific baseline. However, this initial investment yields massive returns in the FG phase. Across all models from 4B to 32B, the total generation time of DDC is consistently and significantly lower than the strongest baseline.
    
    \textbf{Robustness Across Task Difficulty:}
        On the most challenging AIME25 benchmark, DDC not only saves time but also significantly improves accuracy from 73.3\% to 82.1\% on Qwen3-4B. This demonstrates that DDC’s joint optimization approach, coupling inter-path width with intra-path depth, successfully filters out high-confidence hallucinations that often lead other adaptive methods to terminate prematurely on incorrect answers.
    
    The experimental evidence confirms that DDC's analytical complexity is practically zero relative to the LLM's compute. The substantial reduction in total time validates DDC as a highly efficient framework for deploying inference-time scaling in resource-constrained environments.

    \subsection{Sensitivity to Init Budget}
    We investigate the sensitivity of performance to the initialization budget $B_{\text{init}}$, which controls the number of reasoning paths generated during the initialization phase. As summarized in Table~\ref{tab:comprehensive-warmup}, the choice of $B_{\text{init}}$ exhibits a clear trade-off between statistical reliability and computational cost across different models. For Qwen3-4B, this setting yields 88.5\% average accuracy while consuming only 1.04$\times 10^7$ tokens. Smaller initial samples like $B_{\text{init}}=8$ provide insufficient statistical information, leading to premature or incorrect decisions. Larger samples like $B_{\text{init}}=32$ increase computational overhead without corresponding accuracy gains. This demonstrates that a calibrated initial exploration phase is helpful for establishing reliable baselines.
    \begin{table*}[t]
    \caption{Performance comparison across different initialization sizes and models. Acc. denoting accuracy rate and token consumption in units of $10^7$ are reported: WT denotes init tokens, FT denotes final tokens, TT denotes total tokens.}
      \label{tab:comprehensive-warmup}
      \begin{center}
        \begin{small}
        \setlength{\tabcolsep}{3pt}
            \begin{tabular}{@{} >{\raggedright}p{0.65cm} *{14}{c} @{}} 
                \toprule
                \multirow{2}{*}{\textbf{$B_\text{init}$}} & 
                \multicolumn{4}{c}{\textbf{AMC23}} & 
                \multicolumn{4}{c}{\textbf{AIME24}} & 
                \multicolumn{4}{c}{\textbf{AIME25}} &
                \multicolumn{2}{c}{\textbf{Avg}} \\
                \cmidrule(lr){2-5} \cmidrule(lr){6-9} \cmidrule(lr){10-13} \cmidrule(lr){14-15}
                & Acc. $\uparrow$ & WT $\downarrow$ & FT $\downarrow$ & TT $\downarrow$ & 
                Acc. $\uparrow$ & WT $\downarrow$ & FT $\downarrow$ & TT $\downarrow$ & 
                Acc. $\uparrow$ & WT $\downarrow$ & FT $\downarrow$ & TT $\downarrow$ &
                Acc. $\uparrow$ & TT $\downarrow$ \\
                
                \midrule
                \multicolumn{15}{c}{\cellcolor{gray!15}Qwen3-4B} \\
                \midrule
                8 & 97.5 & \textbf{0.26} & 0.54 & 0.80 & 80.0 & \textbf{0.34} & 1.12 & 1.46 & 73.3 & \textbf{0.42} & 2.34 & 2.77 & 83.6 & 1.68 \\
                16 & \textbf{100.0} & 0.39 & \textbf{0.12} & \textbf{0.51} & \textbf{83.3} & 0.69 & 0.47 & \textbf{1.16} & \textbf{82.1} & 0.75 & \textbf{0.71} & \textbf{1.46} & \textbf{88.5} & \textbf{1.04} \\
                32 & 100.0 & 1.04 & 0.00 & 1.04 & 80.0 & 1.40 & \textbf{0.35} & 1.75 & 80.0 & 1.77 & 1.37 & 3.14 & 86.7 & 1.98 \\
                
                \midrule
                \multicolumn{15}{c}{\cellcolor{gray!15}Qwen3-8B} \\
                \midrule
                8 & \textbf{100.0} & \textbf{0.29} & 0.38 & \textbf{0.67} & \textbf{83.3} & \textbf{0.36} & 1.15 & 1.51 & \textbf{76.7} & \textbf{0.43} & 2.27 & 2.69 & \textbf{86.7} & 1.62 \\
                16 & 97.5 & 0.57 & 0.44 & 1.01 & 80.0 & 0.71 & \textbf{0.57} & \textbf{1.29} & 73.3 & 0.86 & 1.31 & \textbf{2.17} & 83.6 & \textbf{1.49} \\
                32 & 100.0 & 1.14 & \textbf{0.18} & 1.32 & 80.0 & 1.44 & 0.71 & 2.15 & 76.7 & 1.79 & \textbf{0.89} & 2.68 & 85.6 & 2.05 \\
                
                \midrule
                \multicolumn{15}{c}{\cellcolor{gray!15}DeepSeek-R1-0528-Qwen3-8B} \\
                \midrule
                8 & 100.0 & 0.38 & \textbf{0.38} & \textbf{0.76} & 83.3 & \textbf{0.36} & 1.15 & \textbf{1.51} & 83.3 & \textbf{0.52} & 1.70 & \textbf{2.22} & 88.9 & \textbf{1.50} \\
                16 & \textbf{100.0} & 0.80 & \textbf{0.00} & 0.80 & \textbf{86.7} & 1.01 & 0.92 & 1.93 & \textbf{83.3} & 1.07 & 1.30 & 2.37 & \textbf{90.0} & 1.70 \\
                32 & 100.0 & 0.92 & 0.00 & 0.92 & 76.7 & 1.48 & \textbf{0.71} & 2.19 & 80.0 & 2.12 & \textbf{1.25} & 3.37 & 85.6 & 2.16 \\
                
                \bottomrule
            \end{tabular}
        \end{small}
      \end{center}
    \end{table*}

    \begin{table*}[t]
    \caption{L(x)=Sliding window with x tokens.}
      \label{tab:window-size}
      \begin{center}
        \begin{small}
        \setlength{\tabcolsep}{3.5pt}
            \begin{tabular}{@{} c *{12}{c} @{}}
                \toprule
                \multirow{2}{*}{\textbf{Window size}} &
                \multicolumn{3}{c}{\textbf{AMC23}} &
                \multicolumn{3}{c}{\textbf{AIME24}} &
                \multicolumn{3}{c}{\textbf{AIME25}} &
                \multicolumn{3}{c}{\textbf{Avg}} \\
                \cmidrule(lr){2-4} \cmidrule(lr){5-7} \cmidrule(lr){8-10} \cmidrule(lr){11-13}
                & Acc. $\uparrow$  & FT $\downarrow$ & TT $\downarrow$ &
                Acc. $\uparrow$  & FT $\downarrow$ & TT $\downarrow$ &
                Acc. $\uparrow$  & FT $\downarrow$ & TT $\downarrow$ &
                Acc. $\uparrow$  & FT $\downarrow$ & TT $\downarrow$ \\
    
                \midrule
                \multicolumn{13}{c}{\cellcolor{gray!15} Qwen3-4B} \\
                \midrule
                L(512) & 100.0 & \textbf{0.03} & 0.54 & 80.0 & \textbf{0.08} & \textbf{0.78} & 76.7 & \textbf{0.15} & \textbf{0.98} & 85.6 & 0.09 & \textbf{0.77} \\
                L(1024) & 100.0 & 0.12 & 0.63 & 80.0 & 0.30 & 1.00 & 73.3 & 0.48 & 1.32 & 84.4 & \textbf{0.30} & 0.98 \\
                L(2048) & \textbf{100.0} & 0.12 & \textbf{0.51} & \textbf{83.3} & 0.47 & 1.16 & \textbf{82.1} & 0.71 & 1.46 & \textbf{88.5} & 0.43 & 1.04 \\
    
                \midrule
                \multicolumn{13}{c}{\cellcolor{gray!15} Qwen3-8B} \\
                \midrule
                L(512) & 97.5 & \textbf{0.00} & \textbf{0.55} & 80.0 & \textbf{0.10} & \textbf{0.83} & \textbf{80.0} & \textbf{0.21} & \textbf{1.06} & \textbf{85.8} & \textbf{0.10} & \textbf{0.81} \\
                L(1024) & 97.5 & 0.15 & 0.71 & \textbf{83.3} & 0.23 & 0.94 & 76.7 & 0.54 & 1.40 & 85.8 & 0.31 & 1.02 \\
                L(2048) & \textbf{97.5} & 0.44 & 1.01 & 80.0 & 0.57 & 1.29 & 73.3 & 1.31 & 2.17 & 83.6 & 0.77 & 1.49 \\
    
                \midrule
                \multicolumn{13}{c}{\cellcolor{gray!15}DeepSeek-R1-0528-Qwen3-8B} \\
                \midrule
                L(512) & 100.0 & 0.00 & 0.76 & 86.7 & 0.10 & 1.09 & 83.3 & \textbf{0.18} & \textbf{1.23} & 90.0 & \textbf{0.09} & \textbf{1.03} \\
                L(1024) & 100.0 & 0.00 & 0.77 & 90.0 & 0.46 & 0.98 & 83.3 & 0.46 & 1.52 & \textbf{91.1} & 0.31 & 1.09  \\
                L(2048) & \textbf{100.0} & \textbf{0.00} & 0.80 & 86.7 & 0.92 & 1.93 & \textbf{83.3} & 1.30 & 2.37 & 90.0 & 0.74 & 1.70 \\
    
                \bottomrule
            \end{tabular}
        \end{small}
      \end{center}
    \end{table*}

    \subsection{Effect of Window Granularity}
    We analyze the window size $L$ used to compute group confidence. Results in Table~\ref{tab:window-size} reveal that larger windows $L=2048$ capture longer reasoning paths, leading to better performance on complex tasks like AIME25 82.1\% vs. 76.7\% with $L=512$ for Qwen3-4B. However, smaller windows $L=512$ can be more efficient for simpler problems. Our analysis confirms that longer context windows help distinguish genuine reasoning trends from fluctuations, improving the reliability of quality assessments.
    
    Together, these analyses validate our parameter choices $B_{\text{init}}=16$, $L=2048$ as providing a reliable default configuration that balances accuracy, efficiency, and robustness across diverse reasoning tasks and models.
    \subsection{Scaling behavior}
    We report the models’s accuracy and token consumption vs voting size on different models and datasets \cref{fig:scaling}.
    
    \begin{figure*}[ht]
      \vskip 0.1in
      \centering
      \includegraphics[width=\textwidth, height=\textheight, keepaspectratio]{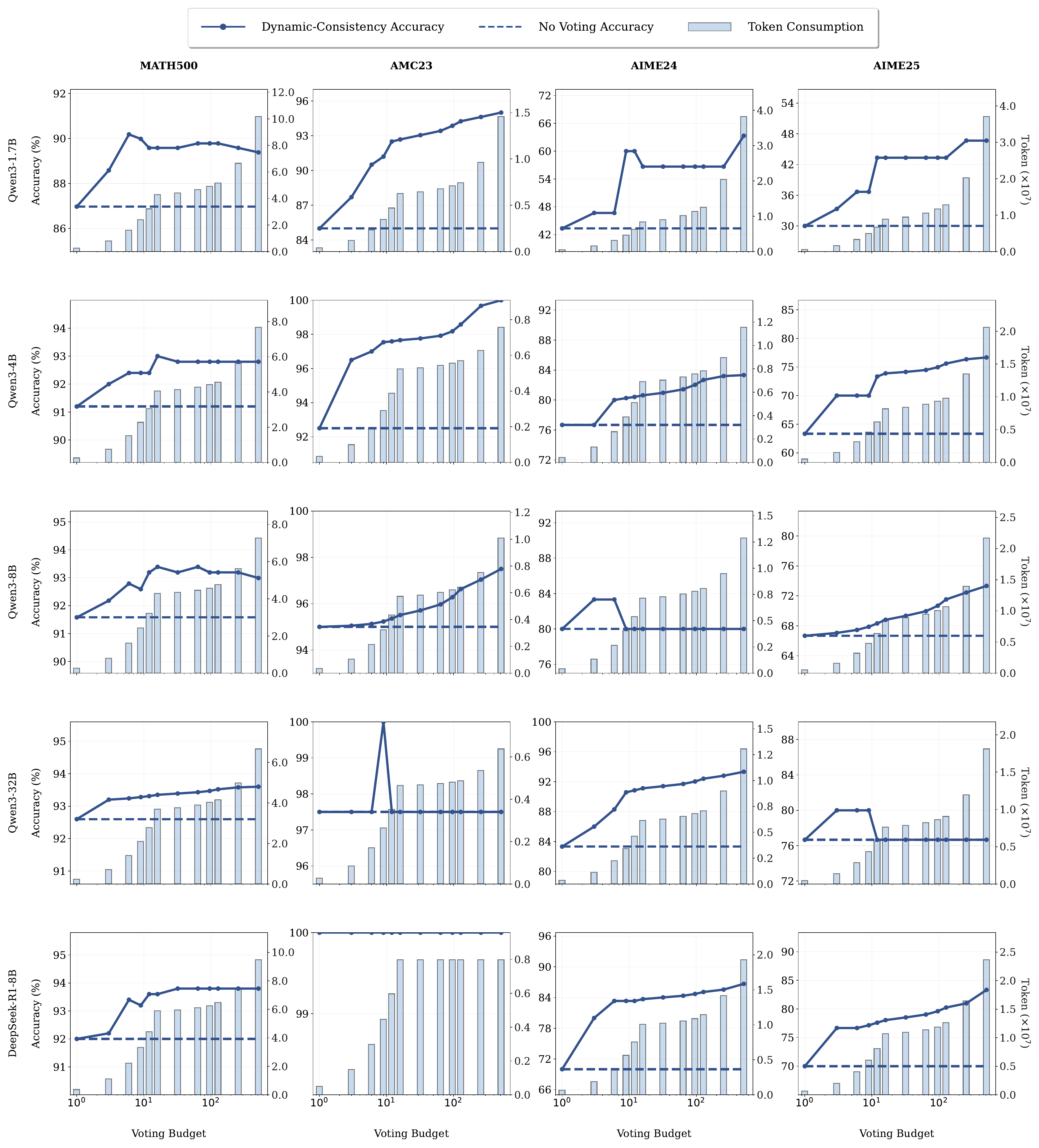}
      \caption{
        Scaling behavior of different models.
      }
      \label{fig:scaling}
    \end{figure*}

    \section{Algorithm}
    \label{algorithm}
    The complete workflow of Dual-Dimensional Consistency is presented in Algorithm~\ref{alg:init_phase} (Initialization Phase), Algorithm~\ref{alg:termination} (Bayesian Termination), Algorithm~\ref{alg:pruning_process} (Stratified Generation), and Algorithm~\ref{alg:main_framework} (Main Framework). The main framework orchestrates the overall process by first performing initialization, then iteratively generating new reasoning paths with termination condition checking, and finally aggregating results.
    
    \begin{algorithm}[tb]
    \caption{Initialization Phase}
    \label{alg:init_phase}
    \begin{algorithmic}[1]
    \REQUIRE Query $x$, Init Budget $B_{\text{init}}$, Window Size $L$
    \ENSURE $\mathcal{S}$, $\mathcal{C}_{\text{local}}$, $\mathcal{C}_{\text{global}}$, $\mathcal{R}_{\text{vals}}$, $\tau_{\text{pass}}$, $\tau_{\text{drop}}$, $\tau_{\text{risk}}$, $\alpha$, $\beta$
    
    \STATE $\mathcal{S} \gets \emptyset$, $\mathcal{C}_{\text{local}} \gets \emptyset$, $\mathcal{C}_{\text{global}} \gets \emptyset$, $\mathcal{R}_{\text{vals}} \gets \emptyset$
    \FOR{$i = 1$ to $B_{\text{init}}$}
        \STATE Generate path $y_i = (y_1, \dots, y_T) \sim P_\theta(\cdot \mid x)$
        
        \STATE Compute $C_t^l = P_t(1)$ and $C_t^g = -\frac{1}{k}\sum_{j=1}^k \log P_t(j)$ for $t \in [1, T]$
        
        \STATE $\forall G_t \in \mathcal{G}(y_i): C_{G_t}^l \gets \frac{1}{L}\sum_{j \in G_t} C_j^l, \quad C_{G_t}^g \gets \frac{1}{L}\sum_{j \in G_t} C_j^g$
        
        \STATE $C_{\text{path}}(y_i) \gets \min_{G_t} C_{G_t}^g$
        
        \STATE Construct matrices $\mathbf{Z}_t$ from window $\{C_j^g\}_{j \in G_t}$ 
        \STATE Compute $\mathcal{R}_t$ via Eigendecomposition of $\mathbf{Z}_t^\top \mathbf{Z}_t$ (Eq.~12)
        
        \STATE $\mathcal{C}_{\text{local}} \gets \mathcal{C}_{\text{local}} \cup \{C_{G_t}^l\}_{\forall t}$, $\mathcal{C}_{\text{global}} \gets \mathcal{C}_{\text{global}} \cup \{C_{G_t}^g\}_{\forall t}$
        \STATE $\mathcal{R}_{\text{vals}} \gets \mathcal{R}_{\text{vals}} \cup \{\mathcal{R}_t\}_{\forall t}$
        \STATE $\mathcal{S} \gets \mathcal{S} \cup \{y_i\}$
    \ENDFOR
    
    \STATE \textit{// Calibrate Query-Specific Thresholds}
    \STATE $\tau_{\text{pass}} \gets \text{Percentile}(\mathcal{C}_{\text{local}}, 90)$
    \STATE $\tau_{\text{drop}} \gets \text{Percentile}(\mathcal{C}_{\text{global}}, 20)$
    \STATE $\tau_{\text{risk}} \gets Q_3(\mathcal{R}_{\text{vals}}) + 1.5 \cdot \text{IQR}(\mathcal{R}_{\text{vals}})$
    
    \STATE \textbf{Initialize Bayesian Termination}
    \STATE $\alpha_u \gets 1, \beta_u \gets 1$ for all $u \in \mathcal{U}$
    \FOR{$y \in \mathcal{S}$}
        \STATE $u' \gets \text{ans}(y)$; $w \gets C_{\text{path}}(y)$
        \STATE $\alpha_{u'} \gets \alpha_{u'} + w$; \quad $\forall u \neq u', \beta_u \gets \beta_u + w$
    \ENDFOR
    
    \STATE \textbf{return} $\mathcal{S}$, $\mathcal{C}_{\text{local}}$, $\mathcal{C}_{\text{global}}$, $\mathcal{R}_{\text{vals}}$, $\tau_{\text{pass}}$, $\tau_{\text{drop}}$, $\tau_{\text{risk}}$, $\alpha$, $\beta$
    \end{algorithmic}
    \end{algorithm}
    
    \begin{algorithm}[tb]
    \caption{Confidence-Weighted Bayesian Termination}
    \label{alg:termination}
    \begin{algorithmic}[1]
    \REQUIRE Advantage counts $\alpha = \{\alpha_u\}$, disadvantage counts $\beta = \{\beta_u\}$
    \ENSURE Boolean indicating whether to terminate
    
    \STATE $u^* \gets \arg\max_{u} \alpha_u$
    \STATE \textit{// Compute the lower bound of the beta distribution's confidence interval}
    \STATE $p_{\text{lower}} \gets I_{0.5}^{-1}(\alpha_{u^*}, \beta_{u^*})$ \textit{// Inverse regularized incomplete beta function}
    \STATE \textit{// Check if the probability that $u^*$ is correct exceeds threshold}
    \IF{$1 - p_{\text{lower}} > 0.95$}
        \STATE \textbf{return} \text{true}
    \ELSE
        \STATE \textbf{return} \text{false}
    \ENDIF
    \end{algorithmic}
    \end{algorithm}
    
    \begin{algorithm}[tb]
    \caption{Trend-Aware Stratified Generation}
    \label{alg:pruning_process}
    \begin{algorithmic}[1]
    \REQUIRE Query $x$, Thresholds $\tau_{\text{pass}}, \tau_{\text{drop}}, \tau_{\text{risk}}$, Window $L$
    \ENSURE Completed path $y$ or NULL
    \STATE $y \gets \emptyset$
    \FOR{$t = 1$ to $T_{\max}$}
        \STATE $y_t \sim P_\theta(\cdot \mid x, y_{<t})$
        \STATE $y \gets y \cup \{y_t\}$
        \IF{$t < L$} \STATE \textbf{continue} \ENDIF
        
        \STATE Define window $G_t = \{y_{t-L+1}, \dots, y_t\}$
        
        \STATE \textit{// Tier 1: Local Confidence Check}
        \STATE $C_{G_t}^l \gets \frac{1}{L} \sum_{j \in G_t} P_\theta(y_j \mid y_{<j})$
        \IF{$C_{G_t}^l > \tau_{\text{pass}}$}
            \STATE \textbf{continue}
        \ENDIF
    
        \STATE \textit{// Tier 2: Global Confidence Check}
        \STATE $C_{G_t}^g \gets -\frac{1}{L} \sum_{j \in G_t} \frac{1}{k}\sum_{v=1}^k \log P_\theta(v \mid y_{<j})$
        \IF{$C_{G_t}^g < \tau_{\text{drop}}$}
            \STATE \textbf{return} \text{NULL}
        \ENDIF
    
        \STATE \textit{// Tier 3: Trend-Aware Analysis}
        \STATE Construct $\mathbf{z}_j = [\hat{C}_j^g, \Delta \hat{C}_j^g]^\top$ for $j \in G_t$
        \STATE $\mathbf{Z} \gets [\mathbf{z}_{t-L+1}, \dots, \mathbf{z}_t]^\top$
        \STATE $\{\lambda_1, \lambda_2\} \gets \text{Eigen}(\frac{1}{L-1}\mathbf{Z}^\top \mathbf{Z})$
        \STATE $\mathcal{R} \gets (1 - \frac{\lambda_1 - \lambda_2}{\lambda_1 + \lambda_2}) + \eta \cdot \mathbb{I}(\text{align} < 0) \cdot \text{align}^2$
        \IF{$\mathcal{R} > \tau_{\text{risk}}$}
            \STATE \textbf{return} \text{NULL}
        \ENDIF
    \ENDFOR
    \STATE \textbf{return} $y$
    \end{algorithmic}
    \end{algorithm}
    
    \begin{algorithm}[tb]
    \caption{Dual-Dimensional Consistency Framework}
    \label{alg:main_framework}
    \begin{algorithmic}[1]
    \REQUIRE Query $x$, Sampling Budget $B$, init Budget $B_{\text{init}}$, Window Size $L$
    \ENSURE Final Answer $\hat{y}$
    
    \STATE \textit{// Initialization Phase}
    \STATE Execute Algorithm~\ref{alg:init_phase} with inputs $x$, $B_{\text{init}}$, $L$ to obtain 
    \STATE $\quad \mathcal{S}$, $\mathcal{C}_{\text{local}}$, $\mathcal{C}_{\text{global}}$, $\mathcal{R}_{\text{vals}}$, $\tau_{\text{pass}}$, $\tau_{\text{drop}}$, $\tau_{\text{risk}}$, $\alpha$, $\beta$
    
    \WHILE{$|\mathcal{S}| < B$}
        \STATE $\text{terminate} \gets \text{ConfidenceWeightedBayesianTermination}(\alpha, \beta)$
        \IF{$\text{terminate}$}
            \STATE \textbf{break}
        \ENDIF
    
        \STATE $y_{\text{new}} \gets \text{TrendAwareStratifiedGeneration}(x, \tau_{\text{pass}}, \tau_{\text{drop}}, \tau_{\text{risk}}, L)$
        
        \IF{$y_{\text{new}} \neq \text{NULL}$}
            \STATE $\mathcal{S} \gets \mathcal{S} \cup \{y_{\text{new}}\}$
            \STATE $u' \gets \text{ans}(y_{\text{new}})$; $w \gets C_{\text{path}}(y_{\text{new}})$
            \STATE $\alpha_{u'} \gets \alpha_{u'} + w$
            \STATE $\forall u \neq u', \beta_u \gets \beta_u + w$
        \ENDIF
    \ENDWHILE
    
    \STATE \textbf{return} $\arg\max_{u} \sum_{y \in \mathcal{S}} \mathbb{I}\{\text{ans}(y) = u\} \cdot C_{\text{path}}(y)$
    \end{algorithmic}
    \end{algorithm}


\end{document}